\documentclass[11pt, copyright, logo, secondlogo]{template/lark}
\usepackage[authoryear, sort&compress, round]{natbib}
\usepackage{bbm}
\usepackage{enumitem}
\setlist{nosep}
\usepackage{xspace}
\usepackage{graphicx}
\usepackage{tabularx}
\usepackage{booktabs}
\usepackage{subcaption}
\usepackage{xcolor}
\usepackage{hyperref}
\usepackage{verbatim}
\usepackage{float}
\usepackage{cleveref}
\usepackage{wrapfig}
\usepackage{makecell}
\usepackage{multirow}
\usepackage{fancyvrb}
\usepackage[normalem]{ulem}
\usepackage{url}
\usepackage{tcolorbox}
\usepackage{listings}
\tcbuselibrary{listings,skins,breakable}

\usepackage{hyperref}       
\usepackage{url}            
\usepackage{booktabs}       
\usepackage{amsfonts}       
\usepackage{nicefrac}       
\usepackage{microtype}      
\usepackage{xcolor}         
\usepackage{microtype}
\usepackage{graphicx}
\usepackage{subcaption}
\usepackage[table]{xcolor}
\usepackage{booktabs} 
\usepackage[dvipsnames]{xcolor}
\usepackage[dvipsnames*,svgnames]{xcolor}
\usepackage{amsmath}
\usepackage{amssymb}
\usepackage{mathtools}
\usepackage{adjustbox}
\usepackage{amsthm}
\usepackage{hyperref}
\usepackage{multirow}
\usepackage{algorithm}
\usepackage{algorithmic}

\theoremstyle{plain}
\newtheorem{theorem}{Theorem}[section]

\newtheorem{lemma}[theorem]{Lemma}

\theoremstyle{definition}
\newtheorem{definition}[theorem]{Definition}
\newtheorem{assumption}[theorem]{Assumption}
\theoremstyle{remark}

\usepackage{amsmath,amsfonts,bm}

\def\eqref#1{equation~\ref{#1}}

\def\1{\bm{1}}

\def\vzero{{\bm{0}}}

\def\vtheta{{\bm{\theta}}}

\def\vh{{\bm{h}}}

\def\vk{{\bm{k}}}

\def\vo{{\bm{o}}}
\def\vp{{\bm{p}}}
\def\vq{{\bm{q}}}

\def\vv{{\bm{v}}}

\def\vy{{\bm{y}}}
\def\vz{{\bm{z}}}

\def\mG{{\bm{G}}}

\def\mI{{\bm{I}}}
\def\mJ{{\bm{J}}}

\def\mR{{\bm{R}}}
\def\mS{{\bm{S}}}

\def\mW{{\bm{W}}}
\def\mX{{\bm{X}}}

\DeclareMathAlphabet{\mathsfit}{\encodingdefault}{\sfdefault}{m}{sl}
\SetMathAlphabet{\mathsfit}{bold}{\encodingdefault}{\sfdefault}{bx}{n}

\def\gZ{{\mathcal{Z}}}

\newcommand{\E}{\mathbb{E}}
\newcommand{\Ls}{\mathcal{L}}
\newcommand{\R}{\mathbb{R}}

\newcommand{\method}{\textsc{QK-Restore }}

\newcommand{\RR}{\mathbb{R}}

\newcommand{\Lattn}{{L}_{\mathrm{attn}}}

\newcommand{\Ltwopi}{L^{2}(\pi)}
\newcommand{\norm}[1]{\left\lVert#1\right\rVert}

\newcommand{\bvbar}{\bar{\vv}}

\usepackage{times}
\usepackage{latexsym}

\usepackage[T1]{fontenc}

\usepackage[utf8]{inputenc}

\usepackage{microtype}

\usepackage{inconsolata}

\usepackage{graphicx}

\title{Attention Amnesia in Hybrid LLMs: When CoT Fine-Tuning Breaks Long-Range Recall, and How to Fix It}

\author[1]{Xinyu Zhou$^{*}$}
\author[2]{Boyu Zhu$^{*}$}
\author[3]{Yi Xu}
\author[1]{Zhiwei Li}
\author[4]{Yingfa Chen}
\author[5]{Huiming Wang$^{\dag}$}
\author[1,6]{Zhijiang Guo$^{\dag}$}

\affil[1]{LARK, HKUST(GZ)}
\affil[2]{UCL}
\affil[3]{Mistral AI}
\affil[4]{Tsinghua University}
\affil[5]{SUTD}
\affil[6]{HKUST}

\correspondingauthor{Huiming Wang (huiming\_wang@mymail.sutd.edu.sg), Zhijiang Guo (zhijiangguo@hkust-gz.edu.cn)}
\githubpage{https://github.com/LARK-AI-Lab/QK-Restore}

\begin{abstract}
Chain-of-thought (CoT) supervised fine-tuning (SFT) is widely adopted to improve reasoning ability, yet we find that it systematically degrades long-context recall in hybrid linear-attention models. Across architectures including HypeNet and Jet-Nemotron, retrieval performance on Needle-In-A-Haystack (NIAH) deteriorates substantially after CoT-SFT, and the degradation becomes more severe under harder retrieval settings and longer context windows. For example, HypeNet-9B on NIAH-\texttt{S2@256K} decreases from $67.2\%$ to $9.4\%$. We attribute this to CoT-SFT biasing attention gradients toward short-range patterns, disrupting query-key projections ($\mW_Q,\mW_K$) that are responsible for long-range routing. Motivated by this observation, we propose \textsc{QK-Restore}, a training-free method that restores only $\mW_Q$ and $\mW_K$ from the pre-SFT checkpoint while preserving all other post-SFT parameters. We further introduce a Procrustes variant to balance routing preservation and reasoning adaptation. Across architectures, \textsc{QK-Restore} consistently restores long-context capability at zero training cost while preserving reasoning performance; for instance, on HypeNet-5B it improves \texttt{S3@256K} from $65.4\%$ to $76.4\%$ while maintaining strong reasoning performance.
\end{abstract}

\begin{document}
\maketitle

\begin{abstract}
Chain-of-thought (CoT) supervised fine-tuning (SFT) is widely adopted to improve reasoning ability, yet we find that it systematically degrades long-context recall in hybrid linear-attention models. Across architectures including HypeNet and Jet-Nemotron, retrieval performance on Needle-In-A-Haystack (NIAH) deteriorates substantially after CoT-SFT, and the degradation becomes more severe under harder retrieval settings and longer context windows. For example, HypeNet-9B on NIAH-\texttt{S2@256K} decreases from $67.2\%$ to $9.4\%$. We attribute this to CoT-SFT biasing attention gradients toward short-range patterns, disrupting query-key projections ($\mW_Q,\mW_K$) that are responsible for long-range routing. Motivated by this observation, we propose \textsc{QK-Restore}, a training-free method that restores only $\mW_Q$ and $\mW_K$ from the pre-SFT checkpoint while preserving all other post-SFT parameters. We further introduce a Procrustes variant to balance routing preservation and reasoning adaptation. Across architectures, \textsc{QK-Restore} consistently restores long-context capability at zero training cost while preserving reasoning performance; for instance, on HypeNet-5B it improves \texttt{S3@256K} from $65.4\%$ to $76.4\%$ while maintaining strong reasoning performance.
\end{abstract}

\section{Introduction}

Efficient sequence models, such as linear attention \citep{katharopoulos2020transformersrnnsfastautoregressive,qin2024variouslengthsconstantspeed,yang2024gatedlinearattentiontransformers} and state-space models \citep{gu2024mambalineartimesequencemodeling,dao2024transformersssmsgeneralizedmodels}, have emerged as attractive alternatives to softmax-attention \citep{vaswani2023attentionneed} for long-context processing, reducing the quadratic cost to linear complexity through compact recurrent or structured states. However, this compression introduces an information bottleneck on recall-intensive tasks that are crucial for long-context understanding \citep{zhang2024hedgehogporcupineexpressive}, such as Needle-in-a-Haystack (NIAH; \citealt{hsieh2024ruler}). Hybrid linear-attention models mitigate this tradeoff by retaining a small set of softmax-attention layers for global recall while converting the remaining layers into efficient linear-attention layers \citep{nvidia2025nvidianemotron3efficient}. Although native hybrid models such as Qwen3.5 \citep{qwen35blog} and Kimi Linear \citep{kimiteam2025kimilinearexpressiveefficient} achieve strong performance and efficiency, pretraining them from scratch remains highly resource-intensive. This cost motivates recent work that converts pretrained softmax-attention into hybrid linear-attention models through distillation, using only a small fraction of pretraining tokens \citep{chen2026hybridlinearattentionright,li2025distilling,hoshino2025radredundancyawaredistillationhybrid}. Existing conversion work primarily focuses on preserving pretraining capabilities during architectural conversion, while the impact of downstream reasoning post-training on long-context retrieval remains underexplored.

\begin{figure}[!t]
    \centering
    \includegraphics[width=0.6\linewidth]{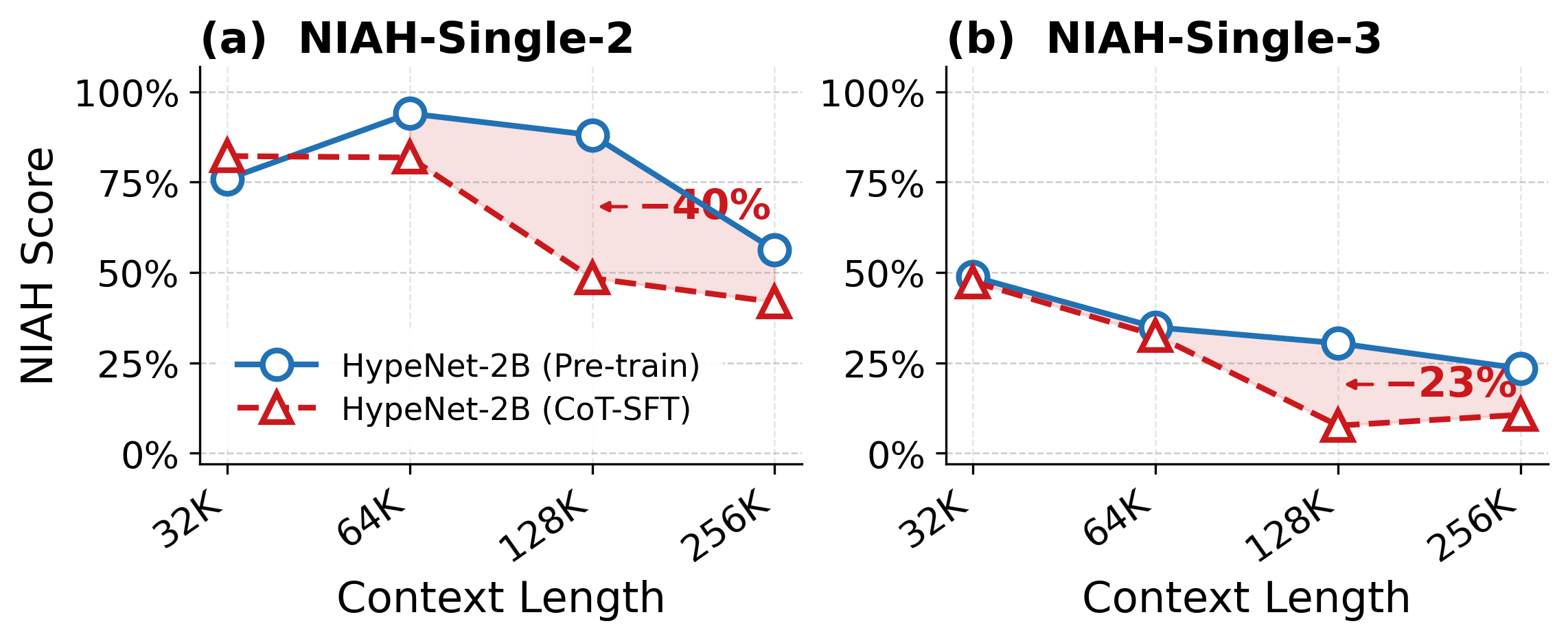}
    \caption{NIAH retrieval accuracy before and after CoT-SFT on HypeNet-2B.
    Despite improving reasoning ability, CoT-SFT substantially degrades long-context retrieval accuracy in distilled hybrid models, especially at longer context lengths and harder NIAH settings.
    }
    \label{fig:niah-sft-degradation}
    \vspace{-2.5mm}
\end{figure}

While distilled hybrid models primarily serve as efficient base models, modern LLMs typically rely on post-training to acquire stronger instruction-following and reasoning abilities \citep{Guo_2025,qwen2025qwen25technicalreport,ouyang2022traininglanguagemodelsfollow}. In particular, CoT-SFT is widely used to enhance mathematical and multi-step reasoning\citep{wei2023chainofthoughtpromptingelicitsreasoning, li2025tl}. We observe that applying CoT-SFT on mathematical reasoning data improves reasoning performance, but can substantially degrade the long-context recall ability that distilled hybrid models acquire during pretraining or architectural conversion, particularly in more challenging NIAH settings and extended contexts length (\autoref{fig:niah-sft-degradation}). Unlike generic catastrophic forgetting, this degradation is highly structured: it primarily affects long-range retrieval behavior mediated by the retained softmax-attention layers, while leaving the intended reasoning gains largely intact. This reveals a fundamental tension in hybrid models: CoT-SFT strengthens local reasoning can simultaneously disrupt the routing mechanisms required for long-range retrieval.

To investigate this phenomenon, we model CoT reasoning traces as a latent Markov process over intermediate reasoning states, reflecting the local step-to-step structure of mathematical derivations \citep{wang2026doeschainofthoughthelpmarkovian,prystawski2023thinkstepstepreasoning}. Under this CoT-Markov assumption, we derive that the expected gradient magnitude on attention logits decays exponentially with token distance. We empirically validate this prediction by measuring token autocorrelation and gradients of attention-logit. These results suggest that CoT-SFT can improve local multi-step reasoning while eroding the long-range routing behavior needed for retrieval.


To better isolate the source of recall degradation, we analyze the retained softmax-attention layers through a routing-extraction decomposition. We find that CoT-SFT induces locality-biased drift in the query-key projections $\mW_Q,\mW_K$, which determine the source of information retrieval, whereas value-side extraction can still benefit from post-SFT adaptation. Since long-context recall in distilled hybrid models depends heavily on a small number of retained softmax-attention layers, such query-key drift can disproportionately disrupt retrieval. Motivated by this asymmetry, we propose \textsc{QK-Restore}, a training-free method that restores only $\mW_Q$ and $\mW_K$ in these layers from the pre-SFT checkpoint while preserving all other post-SFT parameters, thereby recovering long-range routing while retaining most reasoning improvements from CoT-SFT. Our contributions can be summarized as:
\begin{itemize}[leftmargin=*, itemsep=1pt, topsep=2pt, parsep=0pt]
    \item We identify CoT-SFT-induced recall degradation as a structured post-training failure mode of distilled hybrid models.
    \item We provide a theoretical and empirical analysis showing that CoT-SFT concentrates training signals on local token interactions and induces locality-biased drift in the query-key routing geometry of retained softmax-attention layers.
    \item We introduce \textsc{QK-Restore}, a training-free method that restores only query and key projections, recovering long-context recall while largely preserving reasoning performance.
\end{itemize}

\begin{figure*}[!t]
    \centering
    \includegraphics[width=0.98\textwidth]{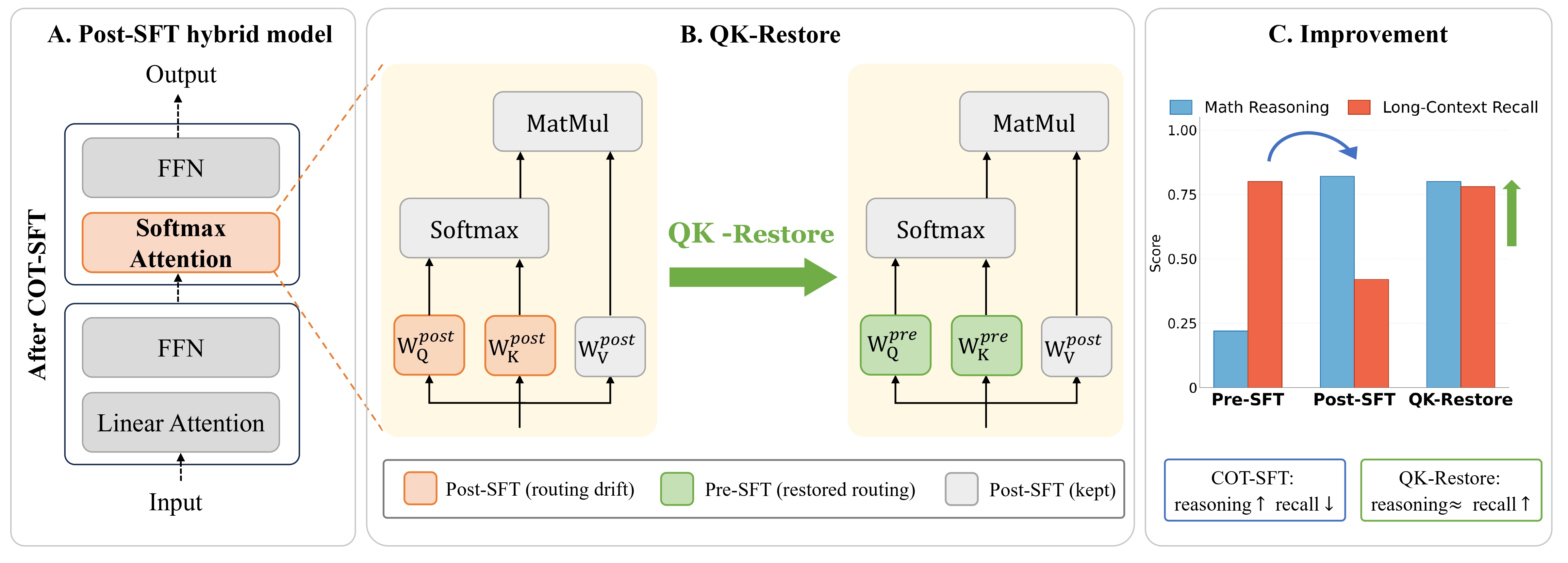}
    \vspace{-2mm}
    \caption{Overview of \textsc{QK-Restore}. After CoT-SFT, the retained softmax-attention layers in a hybrid model may lose long-range routing ability. \textsc{QK-Restore} restores only the query and key projections ($\mW_Q, \mW_K$) from the pre-SFT checkpoint while keeping value-side and other post-SFT parameters unchanged, recovering long-context recall while preserving reasoning performance.}
    \label{fig:qk-restore-overview}
    \vspace{-1.5mm}
\end{figure*}
\section{Related Work}

\noindent\textbf{Efficient Long-Context Models and Hybrid Attention.}
Recent hybrid attention models interleave softmax-attention layers with efficient recurrent or linear-attention layers, achieving competitive performance on both commonsense reasoning and recall-intensive long-context tasks while improving inference efficiency \citep{qwen35blog,kimiteam2025kimilinearexpressiveefficient,nvidia2025nvidianemotron3efficient}. Pretraining strong hybrid models from scratch at large scale remains prohibitively expensive, which motivates recent efforts to obtain hybrid models through Transformer-to-hybrid conversion or distillation \citep{chen2026hybridlinearattentionright,li2025distilling}.  However, existing work primarily focus to obtaining high-performing converted hybrid checkpoints, while their behavior during subsequent post-training remains underexplored.

\noindent\textbf{Distilling Transformers into Hybrid Models.}
Transformer-to-hybrid distillation converts selected softmax-attention layers into linear or recurrent mixers, where layer selection critically influences long-context retrieval performance \citep{goldstein2026radladsrapidattentiondistillation,chen2026hybridlinearattentionright,li2025distilling,gu2026jet}. However, strong recall performance after conversion does not necessarily imply stability after reasoning-oriented post-training; specifically, across multiple distilled hybrid models with strong long-context retrieval, we observe that CoT-SFT can substantially degrade recall capabilities, revealing a critical failure mode that remains unaddressed by existing distillation methods.

\noindent\textbf{CoT Dynamics and Attention Routing.}
Recent work has studied CoT reasoning as a structured generation process over intermediate reasoning states \citep{wang2026doeschainofthoughthelpmarkovian,prystawski2023thinkstepstepreasoning}. 
A Markovian view formalizes this structure by modeling reasoning as local transitions between latent states, suggesting that CoT supervision primarily reinforces short-range step-to-step dependencies \citep{wang2026doeschainofthoughthelpmarkovian}. 
In contrast to prior work that utilizes this locality to explain the efficacy of CoT in reasoning, we connect it to a failure mode in efficient long-context models: locality-biased CoT-SFT can erode long-range attention routing required for recall. 
This perspective aligns with analyses of attention as an information-routing mechanism, where query-key interactions define routing logits that select positions, while value projections determine the content extracted from those positions \citep{vaswani2023attentionneed}.

\section{Background and Problem Setup}
\label{sec:prelim}
\subsection{Hybrid Model}
\label{sec:hybrid-model}
 In this work, we primarily focus on the hybrid model. To retain the expressive power of softmax-attention while improving long-context efficiency, hybrid architectures interleave a small set of softmax-attention layers with efficient linear-attention layers \citep{yang2025zebra, chen2026hybridlinearattentionright}.
 
\noindent\textbf{Softmax-Attention Layers}
\label{sec:full_attn}
For Transformer's layer $\ell$, head $h \in \{1,\ldots,H\}$, the per-head projection matrices $\mW_{Q}^{(\ell,h)},\mW_{K}^{(\ell,h)},\mW_{V}^{(\ell,h)}
\in \RR^{d_h \times d}$ ($d_h = d/H$) transform hidden state $\vh_t^{(\ell)} \in \R^{d}$:
\begin{equation}
\begin{aligned}
  \vq_t^{(\ell,h)} &= \mW_Q^{(\ell,h)} \vh_t^{(\ell)},\quad
  \vk_s^{(\ell,h)} = \mW_K^{(\ell,h)} \vh_s^{(\ell)},\\
  \vv_s^{(\ell,h)} &= \mW_V^{(\ell,h)} \vh_s^{(\ell)},
\end{aligned}
\end{equation}
with scalar logit and attention weight
{\small
\begin{equation}
\begin{aligned}
    e_{ts}^{(\ell,h)} = \frac{(\vq_t^{(\ell,h)})^{\!\top}\vk_s^{(\ell,h)}}{\sqrt{d_h}},
  A_{ts}^{(\ell,h)} = \frac{\exp{e_{ts}^{(\ell,h)}}}{\sum_{u \le t} \exp{e_{tu}^{(\ell,h)}}}.
  \label{eq:attn_score}
\end{aligned}
\end{equation}}
The per-head output vector and layer output are
\begin{equation}
  \bar{\vv}_t = \sum_{s \le t} A_{ts}\,\vv_s,
  \vo_t^{(\ell)} = \mW_{O}\,\mathrm{Concat}_h\!\bigl(\bar{\vv}_t^{(h)}\bigr),
  \label{eq:attn_out}
\end{equation}

\noindent\textbf{Linear-Attention Layers}
\label{sec:lin_attn}
For layer $\ell$, a recurrent state matrix
$\mS_t \in \RR^{d_k \times d_v}$ evolves via
{\small
\begin{equation}
\mS_t^{(\ell)} = \Phi_t^{(\ell)}\!\left(\mS_{t-1}^{(\ell)},\; \vh_t^{(\ell-1)}\right),
\vo_t^{(\ell)} = \Psi_t^{(\ell)}\!\left(\mS_t^{(\ell)},\; \vh_t^{(\ell-1)}\right),
  \label{eq:recurrence}
\end{equation}}
where $\Phi_t$, $\Psi_t$ are layer-specific update and readout operators. Taking Lightning Attention \citep{qin2024lightningattention2freelunch} as an example, the state update and readout are:
\begin{align}
    \mS_t^{(\ell)} &= \mathrm{diag}(\lambda_t^{(\ell)})\, \mS_{t-1}^{(\ell)} + \vk_t^{(\ell)}\bigl(\vv_t^{(\ell)}\bigr)^\top,\\
    \vo_t^{(\ell)} &= \bigl(\mS_t^{(\ell)}\bigr)^\top \vq_t^{(\ell)}.
\end{align}

 We denote the set of softmax-attention layers as ${L}_{\text{attn}}$. Empirical studies suggest that long-range recall in hybrid models depends disproportionately on a set of softmax-attention layers, whereas most remaining layers can be replaced by recurrent mechanisms with minimal degradation \citep{wang2025systematicanalysishybridlinear, chen2026hybridlinearattentionright,jelassi2024repeat}. Thus, preserving the routing behavior of ${L}_{\text{attn}}$ is critical for long-context recall.

\subsection{Attention Routing}

\label{sec:routing_stored}
We operationalize attention-routing through the pre-softmax logit $e_{ts} = \vq_t^{\top}\vk_s / \sqrt{d_h}$, which determines how strongly position $t$ attends to position $s$. Since the routing pattern is determined solely by $\mW_Q,\mW_K$, changes to these matrices directly alter long-range retrieval behavior. To understand how CoT-SFT training affects retrieval behavior, we examine the gradients w.r.t QK metrics. The gradient update of $\mW_Q$ is:
\begin{equation}
  \nabla_{\mW_Q}\Ls
  = \frac{1}{\sqrt{d_h}}\sum_{t,s}
    \frac{\partial\Ls}{\partial e_{ts}}\,\vk_s\,\vh_t^{\!\top}.
  \label{eq:wq_grad}
\end{equation}
If $\partial\Ls/\partial e_{ts}$ is large for small distance $\tau = t{-}s$ and
negligible for large $\tau$, gradient descent systematically pushes
$\mW_Q, \mW_K$ toward local patterns regardless of context length. 
Furthermore, we have:
\begin{equation}
  \frac{\partial\Ls}{\partial e_{ts}}
  = A_{ts}\cdot \mG_t^{\!\top}(\vv_s - \bvbar_t),
  \label{eq:grad_factored}
\end{equation}
where $\mG_t = \partial\Ls/\partial \vo_t \in \R^{d_h}$ is the
{gradient vector}. So understanding the distance dependence of ${\partial\Ls}/{\partial e_{ts}}$ becomes central to understanding how CoT-SFT affects retrieval.





\section{Why CoT-SFT Disrupts Routing?}
\label{sec:generative}
  
\subsection{CoT Data Assumption}
We first characterize CoT data structure via a latent Markov model \citep{wang2026doeschainofthoughthelpmarkovian, prystawski2023thinkstepstepreasoning}, which motivates the gradient analysis that follows.

\begin{assumption}[CoT-Markov structure]
\label{asm:cot_markov}
There exist latent reasoning states $z_1,\ldots,z_T \in \gZ$,
$|\gZ| = K < \infty$, such that: 

\noindent\textbf{1. Markov transitions:} $P(z_t \mid z_{t-1}, z_{t-2},\ldots) = P(z_t \mid z_{t-1})$. 

\noindent\textbf{2. Observation model:} $P(x_t \mid z_t, x_{t-1}) = P(x_t \mid z_t)$. The token at position $t$ is determined by the current reasoning state. 

\noindent\textbf{3. Ergodicity:} The chain is irreducible and aperiodic with stationary distribution $\pi$,
    $\pi_{\min} := \min_{z \in \gZ}\pi(z) > 0$.
    
\noindent\textbf{4. Reversibility:} Detailed balance holds: $\pi(z)\,P(z,z') = \pi(z')\,P(z',z)$ for all $z,z'\in\gZ$. 
    
\noindent\textbf{5. Spectral gap:} The transition matrix has second-largest eigenvalue magnitude $\rho \in (0,1)$.
\end{assumption}
We include the discussion of the validity of this assumption in \autoref{appdix:discuss-cot}.

\subsection{Gradient Locality Theory}
\label{sec:gradient_locality}
Softmax-attention layers $\ell \in \Lattn$ are the sole locus of long-range recall in a hybrid model. Therefore, we analyze how CoT-SFT erodes their routing capacity in these layers through gradient locality.

\subsubsection{Definition and Assumptions}

\begin{definition}[Distance-conditioned gradient magnitude]
\label{def:g_d}
\begin{equation}
  g(\tau) \;:=\; \E_{t,\,\ell \in \Lattn,\,h}
  \!\left[\,\left|\frac{\partial\Ls}{\partial
      e_{t,\,t-\tau}^{(\ell,h)}}\right|\,\right].
\end{equation}
\end{definition}

 \begin{assumption}[Norm bounds]   
  \label{asm:norm} 
  There exist constants $B_G, B_V > 0$ such that, almost surely over the training distribution,   
  \begin{equation}  
    \norm{\mG_t}_2 \;\le\; B_G  
    \quad\text{and}\quad
    \norm{\vv_s}_2 \;\le\; B_V.    
  \end{equation}
  \end{assumption}                                   

\begin{assumption}[Score-function identity]
\label{asm:score}
Let $\mG_t = \partial\Ls/\partial\vo_t$ denote the gradient
of the loss with respect to the per-head output at position $t$. At approximate optimality under cross-entropy, the model satisfies, for all $t$ and almost surely over $x_{1:T}$,
\begin{equation}
  \E\bigl[\mG_t \big|x_{1:t}\bigr] \;=\; \vzero.
  \label{eq:score_identity}
\end{equation}
\end{assumption}

The justifications for the above assumptions are included in \autoref{appdix:norm-bound-assumption} and \autoref{appdix:score-func-assumption}.

\subsubsection{Gradient Locality Theorem}
\begin{lemma}[Spectral Correlation Decay]
\label{lemma:spectral-correclation-decay}
For a stationary reversible ergodic Markov
chain with spectral gap $1-\rho$, and any mean-zero $f, g \in L^2(\pi)$:    
\begin{equation}
    \bigl|\mathbb{E}[f(z_t)\,g(z_s)]\bigr| \;\le\; \rho^{\,t-s}\,\|f\|_{L^2(\pi)}\,\|g\|_{L^2(\pi)},
\end{equation}
\end{lemma}

\begin{theorem}[Gradient locality]
\label{thm:gradient_locality}
Let $B_G, B_V > 0$ be the constants from Assumption~\ref{asm:norm},
and let $\rho \in (0,1)$ be the spectral radius from
Assumption~\ref{asm:cot_markov}.
Under Assumptions~\ref{asm:norm}--\ref{asm:score} and the CoT-Markov
model (Assumption~\ref{asm:cot_markov}), for all $\tau \ge 0$:
\begin{equation}
  \boxed{g(\tau) \;\le\; C'\cdot e^{-\tau/W},}
  \label{eq:main_bound}
\end{equation}
where
\begin{equation}
  C' :=\; 2B_G B_V \;>\; 0,\;
  W :=\; \frac{-1}{\log\rho} \;>\; 0.
  \label{eq:constants}
\end{equation}
\end{theorem}
The constant $C'$ depends only on the model's norm bounds; the decay rate $W$ is set entirely by the spectral gap of the latent transition matrix. The full proof is included in \autoref{appdix:proof-gradient-locality}.

\subsection{Empirical Validation}
Theorem \ref{thm:gradient_locality} shows that CoT training preferentially reinforces nearby token interactions, causing long-range routing signals to decay with distance.
The theoretical chain rests on two empirical claims about data and model:

\noindent\textbf{1. Lemma \ref{lemma:spectral-correclation-decay}'s hypothesis}: the latent Markov chain governing token generation has a spectral gap $1-\rho$. However, since the latent states $z_t$ are unobservable, we instead consider the observed token sequence and use the token autocorrelation $\bar\rho(\tau)$ as a proxy for $\rho$. Concretely, $\bar\rho(\tau)$ measures how much more likely a token is to reappear at distance $\tau$ than by chance, averaged over all token types weighted by their stationary frequency $\pi(v)$. As proved in \autoref{appdix:token-autocorrelation-proof}, $\bar\rho(\tau) \le C\cdot\rho^{\tau}$ under the hidden markov model.


\noindent\textbf{2. Theorem \ref{thm:gradient_locality}'s prediction}: the expected gradient magnitude $\mathbb{E}[|\partial\mathcal{L}/\partial e_{ts}|]$ decays exponentially in $\tau$, with a smaller effective window for CoT data than for general text. We measure $\mathbb{E}[|\partial\mathcal{L}/\partial e_{ts}|]$ at each distance $\tau$ directly from the model:
\begin{align}
    \frac{\partial\mathcal{L}}{\partial e_{ts}}
= A_{ts}\!\left(\frac{\partial\mathcal{L}}{\partial A_{ts}}
  - \sum_{s'} A_{ts'}\frac{\partial\mathcal{L}}{\partial A_{ts'}}\right).
\end{align}


 \begin{figure*}[ht]
      \centering
      \includegraphics[width=0.98\linewidth]{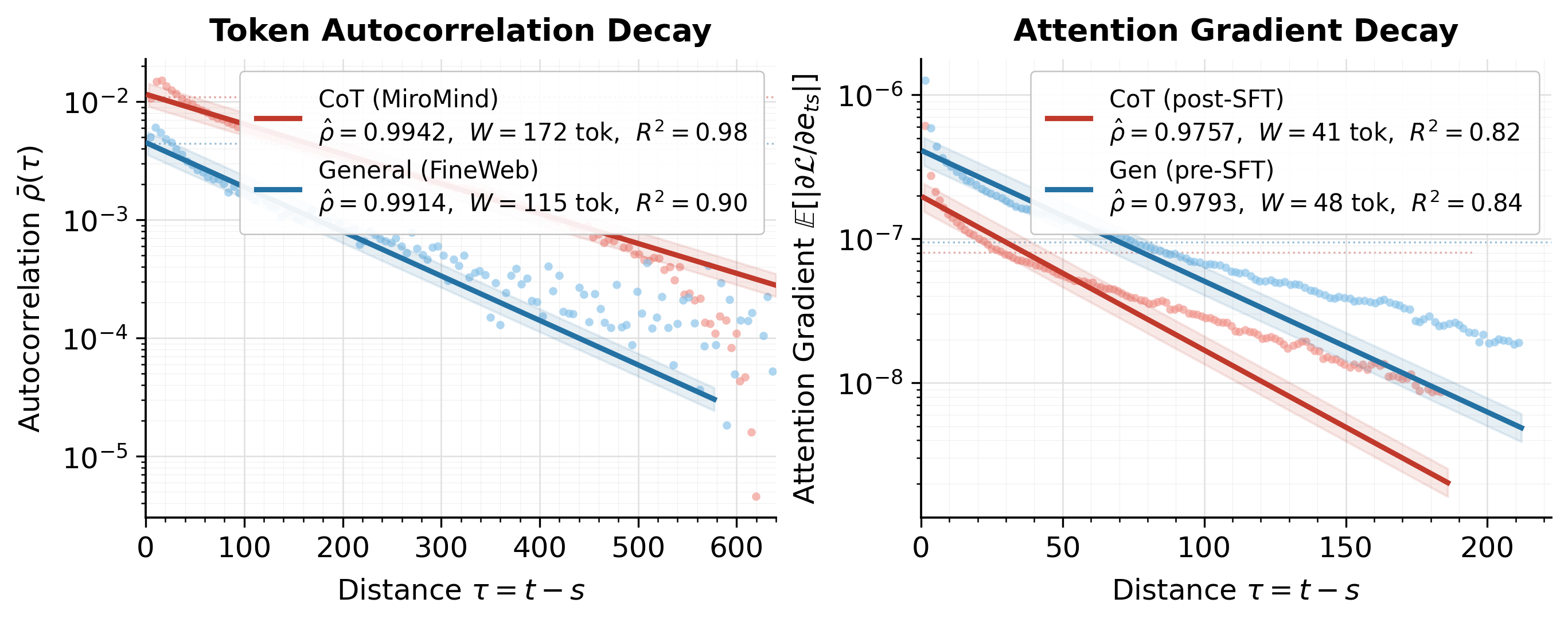}
      \caption{Empirical validation for token autocorrelation decay and attention gradient decay on HypeNet-2B.}
      \label{fig:empirical-validation}
      \vspace{-1mm}
  \end{figure*}

\noindent\textbf{3. Curve Fitting.} Both $\bar\rho(\tau)$ and $g(\tau)$ take the form $C\cdot e^{-\tau/W}$, where $W=-1/\log{\rho'}$ serves as a natural scale of the decay. Therefore, we fit the model $y(\tau) = C \cdot \rho^{\tau} = C \cdot e^{-\tau/W}$ via log-linear regression.  Taking logarithms linearizes the model:
\begin{align}
    \log y(\tau) = \log C - \frac{1}{W}\tau.
\end{align}
\noindent\textbf{4. Interpretation.} As shown in \autoref{fig:empirical-validation}, the left subfigure presents the token autocorrelation decay parameter $W_{\text{corr}}$, a corpus-level statistic that measures the characteristic distance over which text retains structured, non-random self-similarity. Because mathematical reasoning repeatedly references the same symbols, variables, and formula fragments, CoT text sustains high self-similarity over longer distances ($W_{\text{corr}} \approx 172$ tok) than general prose ($W_{\text{corr}} \approx 115$ tok). The right subfigure shows the attention gradient decay parameter $W_{\text{grad}}$, which characterises the training dynamics by measuring the effective reach of the gradient signal $\mathbb{E}[|\partial\mathcal{L}/\partial e_{ts}|]$. Both decay curves are well described by exponential fits, consistent with the thoerem's prediction.


When $W_{\text{grad}} < W_{\text{corr}}$, there exists a band of distances $d \in (W_{\text{grad}},, W_{\text{corr}}]$ where the data demands long-range attention but the training gradient no longer reinforces it. The discussions about the attention gradient decay also observed in the pre-training stage are included \autoref{appdix:discussion-gen-decay}.

\begin{tcolorbox}[
 width=\columnwidth,
  title=\textsc{Key Takeaway-1},
  colback=purple!5!white,
  colframe=red!75!black
]
CoT-SFT improves short-range reasoning ability in hybrid models, at the cost of long-range retrieval.
\end{tcolorbox}

\section{Method}
Theorem \ref{thm:gradient_locality} has established that the gradient on the attention logit $e_{ts}$ decays geometrically in distance $\tau = t - s$. Now we investigate which weight metrics this decay propagates into.

\subsection{Routing–Extraction Gradient Decoupling}
\begin{theorem}[Routing-Extraction Gradient Decoupling]
\label{theom:RE-gradient-decouple}
Under CoT-SFT, where the token-generation process is modeled as a stationary reversible ergodic Markov chain with a spectral gap $1 - \rho > 0$, gradient updates to the parameters of softmax-attention layer exhibit the following asymmetry:

\noindent\textbf{[Routing]} The per-pair contribution to $\nabla_{\mW_Q}$ from position pairs at distance $\tau = t - s$ satisfies:
\begin{equation*}
\boxed{\mathbb{E}\!\left[
  \left\|\frac{\partial \mathcal{L}}{\partial e_{ts}}\, \vk_s\, \vh_t^{\top} \right\|_F
\right] \le C_R \cdot \rho^{\tau} }   
\end{equation*}
for constant $C_R = C' B_K B_h > 0$. The same bound holds for $\nabla_{\mW_K}$. 

\noindent\textbf{[Extraction]} The gradient on the value vector at any position $s$ satisfies:
\begin{equation*}
\boxed{\mathbb{E}\!\left[\left\|\frac{\partial \mathcal{L}}{\partial \vv_s}\right\|\right] \ge \delta_A\cdot  c_G > 0,}
\end{equation*}
where $\delta_A > 0$ satisfies $A_{tt} \ge \delta_A$ for all $t$ and $c_G > 0$ satisfies $\mathbb{E}[\|\mG_s\|] \ge c_G$ for all $s$.
\end{theorem}
The detailed proof is in \autoref{appdix:proof-re-gradient-decouple}. Therefore, CoT-SFT affects routing and knowledge extraction differently. $\mW_Q, \mW_K$ receive meaningful gradient signal only from short-range pairs, while the gradient reaching $\mW_V$ is bounded below uniformly over all positions, independent of context length. 

\subsection{\textsc{QK-Restore}}
Theorem \ref{theom:RE-gradient-decouple} shows that CoT-SFT affects the two functional components of a softmax-attention layer in fundamentally different ways. 

The routing parameters $\mW_Q$ and $\mW_K$, which control {where} the model attends, receive gradient signal that decays as $\rho^{\tau}$ with distance: they are shaped exclusively by short-range pairs and progressively lose long-range routing capacity. The extraction parameters $\mW_V$ and $\mW_O$, which control what is retrieved, receive a gradient uniformly bounded below for every position: they accumulate value-processing improvements.

In other words, CoT-SFT simultaneously {corrupts} $\mW_Q, \mW_K$ and {improves} $\mW_V, \mW_O$. These effects are segregated into disjoint parameter sets. We propose \method, which eliminates routing corruption and preserves extraction
improvement by transplanting $\mW_Q, \mW_K$ from the pre-SFT checkpoint while retaining post-SFT $\mW_V$ and $\mW_O$ (Algorithm~\ref{alg:graft}).
\vspace{-2mm}
 

\begin{algorithm}[t]
\caption{\method}
\label{alg:graft}
\begin{algorithmic}
\STATE {\bfseries Input:} Hybrid model's pre-SFT checkpoint $\vtheta_{\mathrm{pre}}$; Hybrid model's post-SFT checkpoint $\vtheta_{\mathrm{post}}$; softmax-attention layers ${L}_{\mathrm{attn}}$
\STATE{\bfseries Output:}  Repaired model $\vtheta_{\mathrm{rep}}$ 

\STATE \textcolor{gray}{\# \textit{Initialized from post-SFT weights}}
\STATE  $\vtheta_{\mathrm{rep}} \leftarrow \vtheta_{\mathrm{post}}$

\FOR{$\ell \in {L}_{\mathrm{attn}}$}
\STATE \textcolor{gray}{\# \textit{Replace with pre-SFT weights}}
    \STATE $\mW_Q^{\ell,\mathrm{rep}} \leftarrow \mW_Q^{\ell,\mathrm{pre}}, \; \mW_K^{\ell,\mathrm{rep}} \leftarrow \mW_K^{\ell,\mathrm{pre}}$
\ENDFOR

\RETURN $\vtheta_{\mathrm{rep}}$
\end{algorithmic}
\end{algorithm}

\section{Experiments}

\subsection{Setup}
\noindent\textbf{Models.} HypeNet \citep{chen2026hybridlinearattentionright} has recently presented strong long-context recall performance, therefore we focus mainly on this model at scales from 2B to 9B. To ensure a comprehensive analysis, we also include Jet-Nemotron-2B\footnote{\url{https://huggingface.co/collections/jet-ai/jet-nemotron}}.

\noindent\textbf{Training.} To investigate the influence of CoT-SFT, we primarily focus on the math domain, which is widely studied and has sufficient high-quality datasets \citep{ultradata-math, mitra2024orcamath, yu2025dapoopensourcellmreinforcement}. We train HypeNet in different scales on our own in both pre-training and SFT stage, and apply SFT to Jet-Nemotron from their pre-trained checkpoints. More details of datsets and training configuration  are in \autoref{appdix:training-details}. 


\noindent\textbf{Evaluation.} We evaluate the model's performance on long-context recall and math reasoning tasks. For long-context recall, we report the accuracy on NIAH. To measure math reasoning, we test the model on GSM8K \citep{cobbe2021gsm8k} and MATH500 \citep{lightman2023letsverifystepstep}.  


\begin{table*}[!t]
\centering
\caption{NIAH and Math performance comparisons across models after CoT-SFT on the Miromind dataset. {$\uparrow(\downarrow)$} indicates the improvement (degradation) compared to SFT model.} 
\label{tab:exp-results}
\resizebox{0.999\linewidth}{!}{%
\begin{tabular}{ll|llll|llll|llll| cc|cc|cc}
\toprule
\textbf{\textsc{Model}}   & \textbf{\textsc{Method}} & \multicolumn{4}{c}{\textbf{\textsc{NIAH-Single-1}}}  & \multicolumn{4}{c}{\textbf{\textsc{NIAH-Single-2}}}    & \multicolumn{4}{c}{\textbf{\textsc{NIAH-Single-3}}}   & \multicolumn{2}{c}{\textbf{MATH500}} & 
\multicolumn{2}{c}{\textbf{GSM8K}} &\multicolumn{2}{c}{\textbf{AIME24}}  \\
& & \textbf{32K} & \textbf{64K} & \textbf{128K}& \textbf{256K} &  \textbf{32K} & \textbf{64K} & \textbf{128K} & \textbf{256K} & \textbf{32K} & \textbf{64K} & \textbf{128K} & \textbf{256K}  & \texttt{Avg@16} & \texttt{Maj@16} & \texttt{Avg@16} & \texttt{Maj@16} & \texttt{Avg@16} & \texttt{Maj@16}
\\ \midrule

Jet-Nemotron-2B & Pre-train & 100.0	&99.8	&100.0	&100.0 &95.6	&\textbf{91.0}	&\textbf{61.0}	&\textbf{22.2} & \textbf{57.2} &\textbf{48.0}	&\textbf{73.6}	&\textbf{40.6} & 37.5 & 64.0 & 52.6 & 84.3 & 0.80 & 3.30\\
& +SFT &  100.0	&100.0	&100.0	&98.6 &57.8	&28.4	&21.0	&11.6 & 34.4 &33.6	&67.2	&27.8 & \textbf{49.4} & \textbf{70.6} & \textbf{71.9} & 88.0 & 1.50 & 3.30 \\
& +\textbf{\method} & 100.0 & 100.0 & 100.0 & 98.6 & 64.6 & 34.0 & 22.0 & 7.60 & 41.8\color{red}\textsubscript{(7.4$\uparrow$)} & 39.0\color{red}\textsubscript{(5.4$\uparrow$)} & 64.4\color{ForestGreen}\textsubscript{(2.8$\downarrow$)} & 34.0\color{red}\textsubscript{(6.2$\uparrow$)} & 49.0& 69.4 &71.2 & \textbf{88.6} & \textbf{1.90} & \textbf{6.70}\\ \midrule

HypeNet-2B & Pre-train & 99.2	&97.6	&97.2	&98.4 & 75.8	&\textbf{94.0}	&\textbf{88.0}	&\textbf{56.2} & 48.8	&34.8	&30.4	&23.4 & 4.70 & 25.4 & 2.20 & 20.0 & 0.00 & 0.00\\

& +SFT & \textbf{99.8}	&\textbf{99.8}	&99.2	&99.6 & 82.2	&81.8&48.4&41.8 & 47.4	&32.6	&7.60	&10.6 & \textbf{34.4} & 55.2 & 40.8 & 68.1 & 0.00 & 0.00\\
& +\textbf{\method} & 99.4	&99.6	&\textbf{99.4} & \textbf{99.8} & \textbf{97.6}	&91.2	&83.2 & 40.8 & \textbf{56.2}\color{red}\textsubscript{(8.8$\uparrow$)}	&\textbf{52.8}\color{red}\textsubscript{(20.2$\uparrow$)}	&\textbf{30.8}\color{red}\textsubscript{(23.2$\uparrow$)} &\textbf{30.2}\color{red}\textsubscript{(19.6$\uparrow$)} & 33.7 & \textbf{56.6}&  \textbf{41.2} & \textbf{69.0} & \textbf{0.20} & \textbf{3.30}\\ \midrule

HypeNet-5B & Pre-train & 100.0 & 100.0 & 100.0 & 100.0 & 99.4	&99.8	&\textbf{97.0}	&\textbf{93.4} &95.0	&94.2	&87.2	&75.2 & 4.70  & 33.8 & 4.10 & 39.5 & 0.00 & 0.00\\
& +SFT &  100.0 & 100.0 & 100.0 & 100.0 &98.8	&100.0	&91.4	&83.6 & 96.6	&\textbf{96.6}	&\textbf{90.6}	&65.4 & \textbf{49.3} & 67.0 & \textbf{61.3} & 81.4 & 0.60 & 0.00\\
& +\textbf{\method} &100.0 & 100.0 & 100.0 & 100.0 &\textbf{99.8} &\textbf{100.0}	&96.8	&85.4 & \textbf{97.0}\color{red}\textsubscript{(0.4$\uparrow$)}		&96.4\color{ForestGreen}\textsubscript{(0.2$\downarrow$)}		&90.2\color{ForestGreen}\textsubscript{(0.4$\downarrow$)}	&\textbf{76.4}\color{red}\textsubscript{(11.0$\uparrow$)}	 & 48.7 & 68.4 &  60.9 & \textbf{81.8} & 0.40 & \textbf{3.30}\\ \midrule

HypeNet-9B & Pre-train & 100.0 & 100.0 & 100.0 & 100.0 & 64.2	&95.8	&\textbf{72.4} &  \textbf{67.2} & 53.0	&95.0	&\textbf{83.8} &\textbf{52.0} & 17.6 & 57.6 & 15.3 & 62.4 & 0.80 & 6.70\\
& +SFT & 100.0 & 100.0 & 100.0 & 100.0 & 89.0	&64.2	&17.4 & 9.40 & 98.2	&90.0	&47.8 & 22.8 & \textbf{39.8} & \textbf{69.5}  & \textbf{62.3} & 87.8 & \textbf{1.20} & \textbf{10.0}\\
& +\textbf{\method }&100.0 & 100.0 & 100.0 & 100.0 & \textbf{92.6}	&\textbf{100.0}	&44.0 & 19.6 & \textbf{98.6}\color{red}\textsubscript{(0.4$\uparrow$)}	&\textbf{97.0}\color{red}\textsubscript{(7.0$\uparrow$)}	&66.4\color{red}\textsubscript{(18.6$\uparrow$)} & 42.6\color{red}\textsubscript{(19.8$\uparrow$)} & 37.6 & 68.4 & 59.3 & \textbf{88.4} & 0.60 & {3.30}

\\\bottomrule
\end{tabular}
}
\end{table*}

\subsection{Main Results}
The performance comparisons are in \autoref{tab:exp-results}.

\noindent\textbf{CoT-SFT drives long-context degradation.}
A key observation is that CoT-SFT often weakens long-context retrieval abilities acquired during pre-training, particularly under more challenging retrieval settings and longer context windows. While simpler tasks such as NIAH-Single-1 remain nearly saturated, substantial degradation emerges on NIAH-Single-2 and NIAH-Single-3.

For example in \autoref{tab:exp-results}, HypeNet-2B on NIAH-Single-3 at 128K drops from $30.4$ to $7.60$ after SFT, while HypeNet-9B on NIAH-Single-2 at 256K decreases from $67.2$ to $9.40$. Moreover, the degradation consistently becomes more severe as context length increases, suggesting that SFT disproportionately disrupts the mechanisms responsible for long-range retrieval.


\noindent\textbf{\textsc{QK-Restore} recovers long-context capability.}
\textsc{QK-Restore} consistently mitigates the loss of long-context performance incurred during SFT, with the largest gains observed in configurations where the degradation is most severe.

Specifically in \autoref{tab:exp-results}, for HypeNet-2B, NIAH-Single-3 at 256K improves from $10.6$ to $30.2$ (\textbf{+19.6}), while HypeNet-9B increases from $22.8$ to $42.6$ (\textbf{+19.8}). Under OpenThoughts-3 CoT-SFT, HypeNet-9B improves from $22.4$ to $45.0$ on NIAH-Single-2 at 256K and from $40.6$ to $61.2$ on NIAH-Single-3 at 256K. In several configurations, \method even surpasses the original pre-training baseline. For example, HypeNet-2B on NIAH-Single-3 at 64K improves from $34.8$ (pre-train) to $52.8$. This suggests that our approach does not simply recover pre-training behavior; rather, it can preserve retrieval structures from pre-training while still benefiting from task-specific representations introduced during SFT.

\noindent\textbf{\textsc{QK-Restore} preserves downstream reasoning performance.}
Beyond retrieval, we evaluate whether restoring routing parameters affects downstream reasoning capabilities. Across benchmarks, \textsc{QK-Restore} largely preserves the gains obtained from CoT-SFT while substantially improving long-context retrieval.

On mathematical reasoning tasks, performance remains close to the post-SFT model. For example, on HypeNet-5B, \textsc{QK-Restore} incurs only minor changes of $-0.6$ and $-0.4$ points on MATH500 and GSM8K, respectively, while recovering long-context retrieval performance. Similar trends are observed across other model scales. These results suggest that routing parameters can be selectively restored to recover long-context retrieval without significantly affecting the task-specific knowledge and reasoning capabilities acquired during post-training.

\noindent\textbf{Retrieval difficulty amplifies long-context degradation.}
We further observe that retrieval complexity strongly influences robustness. NIAH-Single-1 remains nearly saturated across training stages, whereas substantially larger gaps emerge in Single-2 and Single-3. This trend indicates that difficult retrieval settings place greater demands on precise long-range token interactions and are therefore more sensitive to the degradation.

\begin{table*}[!t]
\centering
\caption{NIAH comparisons across models after Instruction-following (non-CoT) SFT on the Tulu-3 dataset.} 
\label{tab:exp-results2}
\resizebox{0.999\linewidth}{!}{%
\begin{tabular}{ll|llll|llll|llll|c}
\toprule
\textbf{\textsc{Model}}   & \textbf{\textsc{Method}} & \multicolumn{4}{c}{\textbf{\textsc{NIAH-Single-1}}}  & \multicolumn{4}{c}{\textbf{\textsc{NIAH-Single-2}}}    & \multicolumn{4}{c}{\textbf{\textsc{NIAH-Single-3}}} & \textbf{IFEval}    \\
& & \textbf{32K} & \textbf{64K} & \textbf{128K}& \textbf{256K} &  \textbf{32K} & \textbf{64K} & \textbf{128K} & \textbf{256K} & \textbf{32K} & \textbf{64K} & \textbf{128K} & \textbf{256K} 
\\ \midrule

HypeNet-5B & Pre-train & 100.0 & 100.0 & 100.0 & 100.0 & 99.4	&99.8	&{97.0}	&{93.4} &95.0	&94.2	&\textbf{87.2}	&75.2 & 12.0\\
& +SFT &  99.8 & 100.0 & 100.0 & 100.0 & \textbf{99.8} & {100.0} & 99.6 & \textbf{93.6} & \textbf{98.6} & \textbf{97.2} & 85.4 & 83.8 & 26.6\\
& +\textbf{\method} & 100.0 & 100.0 & 100.0 & 100.0 & 99.6 & \textbf{100.0} & \textbf{99.8} & 93.4 & 97.4 & 95.2 & 85.4 & \textbf{85.2} & \textbf{26.6}\\ \midrule

HypeNet-9B & Pre-train & 100.0 & 100.0 & 100.0 & 100.0 & 64.2	&95.8	&{72.4} &  {67.2} & 53.0	&95.0	&{83.8} &{52.0} & 14.0\\
& +SFT & 100.0 & 100.0 & 100.0 & 100.0 & 86.4 & \textbf{97.8} & 88.2 & 78.4 & \textbf{96.2} & 96.8 & \textbf{95.8} & 86.8 & 27.2\\
& +\textbf{\method } & 100.0 & 99.8 & 100.0 & 100.0 & \textbf{88.2} & 97.4 & \textbf{92.8} & \textbf{78.6} & 96.0 & \textbf{97.4} & 95.2 & \textbf{88.2} & \textbf{29.0}

\\\bottomrule
\end{tabular}
}
\end{table*}

\section{Analysis}
\subsection{Analysis on Non-CoT SFT}
Our theoretical analysis attributes long-context degradation to the localized optimization dynamics induced by CoT supervision. A natural question is whether this phenomenon arises from post-training in general, or whether it is specific to CoT-style reasoning traces. To answer this question, we compare CoT-SFT with conventional instruction-following SFT using the Tulu-3 dataset.

As shown in \autoref{tab:exp-results2}, standard instruction tuning does not exhibit the same degradation pattern observed under CoT-SFT. Across both HypeNet-5B and HypeNet-9B, long-context retrieval performance is largely preserved and frequently improved after Tulu-3 training. For example, HypeNet-5B improves from $75.2$ to $83.8$ on NIAH-Single-3 at 256K, while HypeNet-9B improves from $52.0$ to $86.8$. Similar gains are observed on NIAH-Single-2 across multiple context lengths.

Moreover, instruction-following ability improves substantially after Tulu-3 tuning, as measured by IFEval. Despite these improvements, long-context retrieval remains intact. This contrasts sharply with the behavior of CoT-SFT, where retrieval performance deteriorates under identical architectures and context lengths.

These findings suggest that long-context degradation is not an inherent consequence of post-training. Instead, it is closely associated with CoT-style supervision, supporting our hypothesis that the sequential dependency structure of reasoning traces concentrates gradient updates on short-range interactions and disproportionately affects long-range routing behavior.

\begin{tcolorbox}[
 width=\columnwidth,
 title=\textsc{Key Takeaway-2},
 colback=purple!5!white,
 colframe=red!75!black
]
Long-context degradation is not an inherent consequence of post-training, where non-CoT instruction-following SFT preserves or even improves retrieval for hybrid models.
\end{tcolorbox}

\begin{table*}[!t]
\centering
\caption{NIAH comparisons across models after CoT-SFT on the OpenThoughts-3 dataset.} 
\label{tab:exp-results3}
\resizebox{0.999\linewidth}{!}{%
\begin{tabular}{ll|llll|llll|llll|c}
\toprule
\textbf{\textsc{Model}}   & \textbf{\textsc{Method}} & \multicolumn{4}{c}{\textbf{\textsc{NIAH-Single-1}}}  & \multicolumn{4}{c}{\textbf{\textsc{NIAH-Single-2}}}    & \multicolumn{4}{c}{\textbf{\textsc{NIAH-Single-3}}} & \textbf{LCB-V5}   \\
& & \textbf{32K} & \textbf{64K} & \textbf{128K}& \textbf{256K} &  \textbf{32K} & \textbf{64K} & \textbf{128K} & \textbf{256K} & \textbf{32K} & \textbf{64K} & \textbf{128K} & \textbf{256K} & \texttt{Avg@8}
\\ \midrule

HypeNet-5B & Pre-train & 100.0 & 100.0 & 100.0 & 100.0 & 99.4	&99.8	&{97.0}	&\textbf{93.4} &95.0	&94.2	&{87.2}	&75.2 & 0.27\\
& +SFT &  100.0 & 100.0 & 99.8 & 100.0 & 99.8 & 100.0 & 99.0 &  89.6 & \textbf{98.8} & \textbf{97.8} &\textbf{94.2} & \textbf{91.4} & 6.37 \\
& +\textbf{\method} & 100.0 & 99.8 & 99.8 & 100.0 & \textbf{99.8} & \textbf{100.0} & \textbf{100.0} & 92.8 & 98.6 & 97.4 & 93.2 & 79.4 & \textbf{6.92}\\ \midrule

HypeNet-9B & Pre-train & 100.0 & 100.0 & 100.0 & 100.0 & 64.2	&95.8	&{72.4} &  \textbf{67.2} & 53.0	&95.0	&{83.8} &{52.0} & 1.66 \\
& +SFT &  100.0 & 100.0 & 100.0 & 100.0 & 91.8 & 100.0 & 41.8 & 22.4 & \textbf{96.2} & \textbf{98.6} & 76.6 & 40.6 & \textbf{13.11}\\
& +\textbf{\method } & 100.0 & 100.0 & 100.0 & 99.8 & \textbf{94.2} & 99.8 & \textbf{78.2} & 45.0 & 95.0 & 97.4 & \textbf{90.8} & \textbf{61.2} & 12.91

\\\bottomrule
\end{tabular}
}
\end{table*}

\subsection{Analysis Beyond Mathetical Reasoning}
Our primary experiments focus on mathematical reasoning, where high-quality CoT supervision is readily available. To evaluate whether the observed retrieval–reasoning trade-off extends beyond mathematics, we conduct additional experiments using OpenThoughts-3 and evaluate downstream coding performance on LiveCodeBench-V5.

As shown in \autoref{tab:exp-results3}, the same qualitative behavior persists in the coding domain. CoT-SFT substantially improves coding capability, increasing LiveCodeBench-V5 from $0.27$ to $6.37$ on HypeNet-5B and from $1.66$ to $13.11$ on HypeNet-9B. However, these gains are accompanied by noticeable degradation in long-context retrieval, particularly on challenging NIAH settings and long context lengths.

Applying \textsc{QK-Restore} recovers a large fraction of the lost retrieval capability while preserving downstream coding performance. For example, on HypeNet-5B, LiveCodeBench-V5 further improves from $6.37$ to $6.92$, while NIAH-Single-3 at 256K recovers from $75.2$ to $79.4$. Similarly, HypeNet-9B improves from $40.6$ to $61.2$ on NIAH-Single-3 at 256K while maintaining competitive coding performance ($13.11 \rightarrow 12.91$).

These results suggest that the retrieval degradation induced by CoT-SFT is not specific to mathematical reasoning. Instead, it appears to arise from a more general optimization phenomenon associated with long reasoning traces. The effectiveness of \textsc{QK-Restore} across both math and code domains further supports our central hypothesis that long-range retrieval is primarily governed by routing parameters, while task-specific reasoning capabilities are largely encoded elsewhere in the network.

\subsection{Analysis on Attention Routing Recover}

In this section, we visualize the attention map to characterize how CoT supervised fine-tuning alters the long-range routing behavior of the softmax-attention layers in HypeNet-5B. We conduct the analysis on the models trained on Miromind dataset.

In detail, we compare the attention matrices of a pre-SFT and a post-CoT-SFT checkpoint of HypeNet-5B across all softmax-attention layers. We then identify the most affected layer via a per-head mean attention distance metric:
\begin{align}
    \bar{d}_h = \mathbb{E}_{t}\!\left[\sum_{s\leq t} (t{-}s)\,A_{h,t,s}\right],
\end{align}
where $h$ indexes the attention head and $A_{h,t,s}$ is the post-softmax weight from query $t$ to key $s$. Ranking layers by the number of heads whose $\bar{d}_h$ decreases after SFT, Layer 33 emerges as the most affected, with 25 of 32 heads shifting toward shorter attention distances. We therefore focus our analysis on this layer.

\begin{figure*}[!t]
    \centering
    \includegraphics[width=0.99\linewidth]{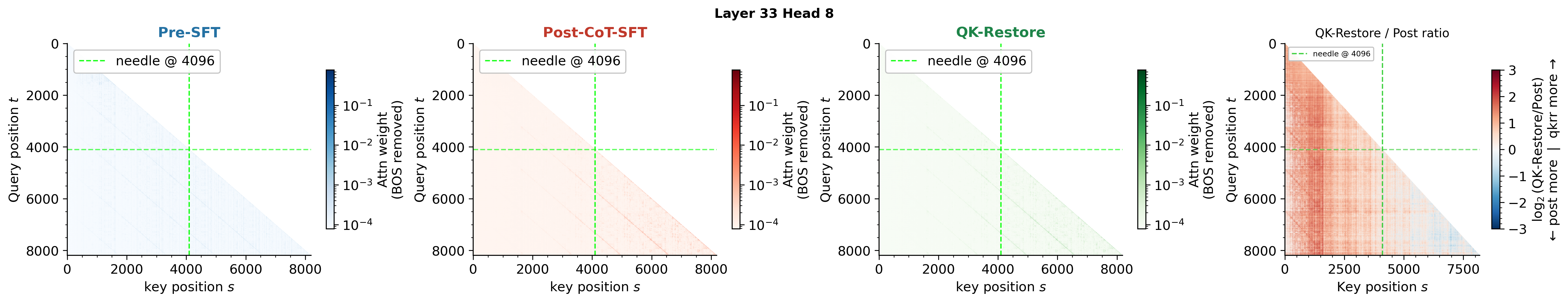}
    \caption{Attention map from query $t$ to key $s$ of \texttt{layer33-head8} in pre-SFT and post-SFT on HypeNet-5B. Particularly, panel 4 demonstrates post-SFT assigning higher weights to nearby tokens, which is consistent with our Theorem, and the effectiveness of \method in restoring the long-range attention behavior.}
    \label{fig:attention-map}
    \vspace{-1mm}
\end{figure*}

We visualize the top recall heads on a synthetic NIAH prompt of 8,192 tokens, with the target fact embedded at $50\%$ context depth. For each head, we display a four-panel figure: the pre-SFT attention matrix $\tilde{A}_\text{pre}$ (blue), the post-SFT matrix $\tilde{A}_\text{post}$ (red), the \method matrix $\tilde{A}_\text{QK-Restore}$ (green), and the element-wise log-ratio $\log_2(\tilde{A}_\text{QK-Restore} / \tilde{A}_\text{post})$ (Panel 4), where positive values indicate positions where \method recovers more long-range attention than post-SFT.

\noindent\textbf{Analysis.} \autoref{fig:attention-map} shows the pattern of Layer-33 Head-8 of HypeNet-5B. From panel 4, we observe that the off-diagonal lower triangle is predominantly red, confirming that \method partially recovers the long-range routing suppressed by CoT fine-tuning: for most query positions $t$, \method attends more strongly to distant keys $s \ll t$ than post-SFT does. Second, the lower-right block, where both the query and key positions fall after the needle $s^*$, shows a blue band along the main diagonal, meaning that post-SFT assigns higher weight to nearby tokens in this region, which matches our hypothesis. These observations demonstrate the effectiveness of \method in mitigating routing collapse and restoring the long-range attention behavior characteristic.

\subsection{Ablation on \method}
\method transplanting the full $\mW_Q, \mW_K$ matrices from the pre-SFT checkpoint into the post-SFT model, recovers long-range recall substantially but incurs a math performance drop. The post-SFT $\mW_Q$ can be decomposed as:
\begin{equation*}
    \mW_Q^{\rm post} = \mW_Q^{\rm pre} + \delta \mW_Q^{\rm route} + \delta \mW_Q^{\rm math},
\end{equation*}
where $\delta \mW_{\rm route}$ encodes harmful locality drift and $\delta \mW_{\rm math}$ encodes beneficial math ability adaptation. Since both components are entangled within $\mW_Q$, \method discards them together, recovering routing at the cost of erasing the math benefit. As expected, \autoref{fig:routing-drift} and \autoref{fig:routing-drift-jet} show a non-negligible drift in $\mR=\mW_Q\mW_K^{\top}$, where $\Delta \mR = \mR^{\rm post}-\mR^{\rm pre}$. 

\begin{figure}[t]
    \centering
    \includegraphics[width=0.5\linewidth]{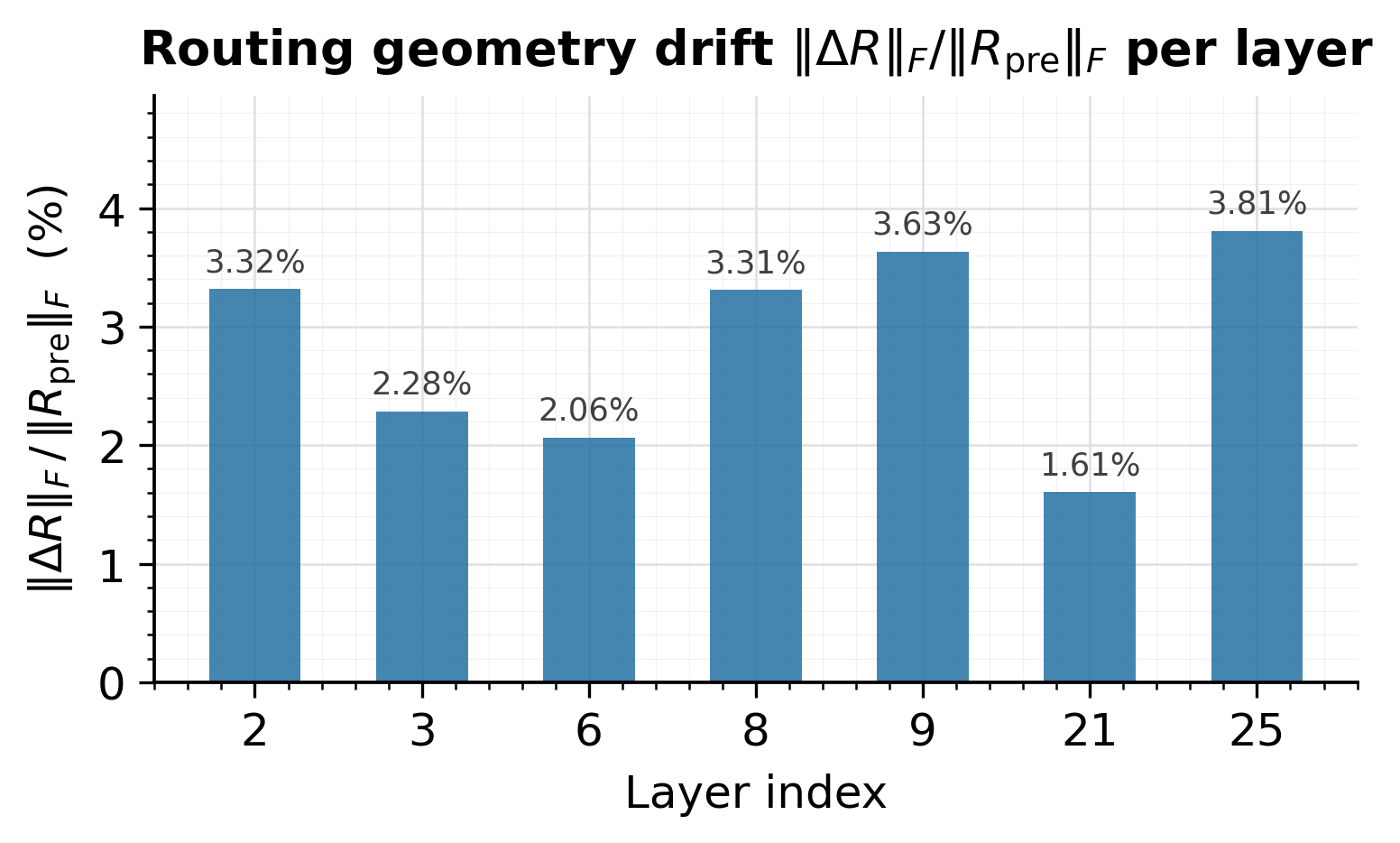}
    \caption{Routing geometry drift in HypeNet-2B.}
    \label{fig:routing-drift}
    \vspace{-2.5mm}
\end{figure}

\begin{figure}[t]
    \centering
    \includegraphics[width=0.5\linewidth]{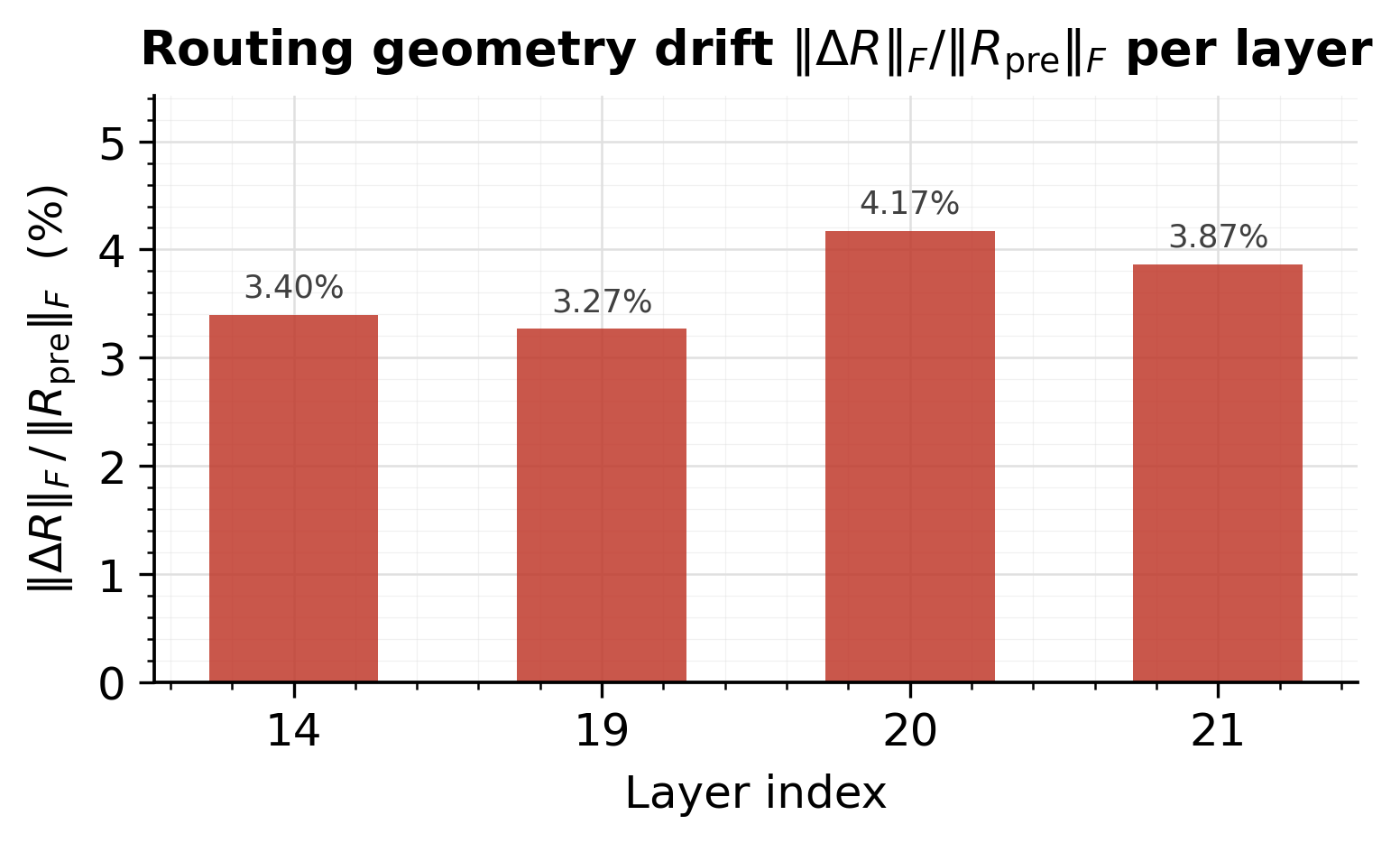}
    \caption{Routing geometry drift in Jet-Nemotron-2B.}
    \label{fig:routing-drift-jet}
    \vspace{-2.5mm}
\end{figure}

Therefore, we model it as a Procrustes problem:
\begin{equation}
\begin{adjustbox}{max width=0.99\linewidth}
$
\begin{aligned}
\min_{\mW_Q^{\rm new},\mW_K^{\rm new}} \quad
\left\|\mW_Q^{\rm new}-\mW_Q^{\rm post}\right\|_F^2 &+
\left\|\mW_K^{\rm new}-\mW_K^{\rm post}\right\|_F^2 \\
\text{s.t.} \
\mW_Q^{\rm new}\mW_K^{{\rm new}\top}
&=
\mR_{\rm pre},
\end{aligned}
$
\end{adjustbox}
\end{equation}
and we call it as \textsc{QK-Pro} variant (more details are in \autoref{appdix:graft-pro}). 

Furthermore, we evaluate the roles of each components in attention (i.e., $\mW_K,\mW_Q,\mW_V$), and CoT-SFT with $\mW_Q^{\rm pre}, \mW_K^{\rm pre}$ frozen as a preventive counterpart (denoted as \textsc{QK-Frozen}), to answer a question: is it preferable to prevent routing drift during training, or to correct it post-hoc?

\begin{table}[!t]
\centering
\caption{Ablation of \method in HypeNet-2B.} 
\label{tab:ablation-qk}
\resizebox{0.6\linewidth}{!}{%
\begin{tabular}{l|lllll}
\toprule
 \textbf{\textsc{Method}} & \multicolumn{2}{c}{\textbf{\textsc{NIAH-Single-2}}}    & \multicolumn{2}{c}{\textbf{\textsc{NIAH-Single-3}}} & \textbf{MATH500}  \\
 &  \textbf{64K} & \textbf{128K}  & \textbf{64K} & \textbf{128K} & \texttt{Avg@16}
\\ \midrule
 \method &\textbf{91.2} & \textbf{83.2} & \textbf{52.8} & \textbf{30.8} & 33.7 \\ \midrule
\textsc{QK-Pro} & 94.0 & 79.0 & 51.8 & 26.0 & \textbf{34.2} \\ 
\textsc{Q-Restore} & 78.4 & 79.0 & 50.8 & 24.4 & 34.2 \\
\textsc{K-Restore} & 90.4 & 71.6 & 46.2 & 13.2 & 34.2\\
\textsc{V-Restore} & 57.2 & 36.6 & 24.0 & 15.2 & 32.1 \\
\textsc{QK-Frozen} & 90.6 & 76.4 & 49.4 & 23.6 & 33.1
\\\bottomrule
\end{tabular}
}
\end{table}

\noindent\textbf{Analysis.} \autoref{tab:ablation-qk} validates three key claims. \textsc{V-Restore} fails to recover NIAH while degrading MATH500, confirming that $\mW_V$ does not contribute to routing recovery. \textsc{Q-Restore} and \textsc{K-Restore} each achieve partial recall recovery but fall significantly short at 128K (e.g., \textsc{K-Restore} collapses on NIAH-Single-3 (13.2 vs.\ 30.8), indicating that isolated restoration introduces a $\mW_Q$-$\mW_K$ mismatch and that joint restoration is necessary for coherent routing geometry. \textsc{QK-Pro} achieves the highest performance on MATH500, but trading long-range routing recovery for math preservation. \textsc{QK-Frozen} consistently underperforms \method in both NIAH and mathematical tasks, indicating that unconstrained SFT is for better $\mW_V$ adaptation.

\section{Conclusion}
\vspace{-1mm}

In this work, we identify CoT-SFT-induced long-context recall degradation as a critical post-training failure mode in distilled hybrid models. We show that reasoning-oriented SFT biases optimization toward short-range attention patterns, perturbing the query-key routing geometry of retained softmax-attention layers and weakening long-range retrieval. Motivated by this observation, we propose \textsc{QK-Restore}, a training-free method that restores only the query and key projections from the pre-SFT checkpoint while preserving all other parameters. Experiments across multiple architectures show that \textsc{QK-Restore} substantially recovers long-context capability while largely preserving SFT reasoning gains. Overall, our findings highlight routing stability as a key factor for efficient long-context models that jointly support retrieval and reasoning.

\clearpage 

\bibliography{custom}
\bibliographystyle{abbrvnat}

\clearpage

\appendix
\section{Discussion of the CoT-Markov Structure Assumption}
\label{appdix:discuss-cot}

\noindent\textbf{Discussion 1.} The 1-step Markov condition is an idealisation for analysis. Any $k$-step Markov chain is equivalent to a 1-step chain on a $K^k$-state product space, yielding the same exponential decay with a larger reasoning horizon $W$. 

\noindent\textbf{Discussion 2.}  Reversibility is invoked in the proof of Lemma \ref{lemma:spectral-correclation-decay} at one step: it implies that the Markov operator $P$ is self-adjoint on $L^2(\pi)$. We acknowledge that a causally generated sequence such as a CoT reasoning chain does not, in general, satisfy detailed balance: the latent transition probabilities are asymmetric by construction. Reversibility is, however, a \emph{sufficient} condition rather than a necessary one. For a general ergodic chain (possibly non-reversible), the identical bound holds with $\rho$ replaced by the second-largest \emph{singular value} of $P$, obtained
via its singular value decomposition in place of the eigendecomposition. Since singular values are always real and non-negative, the proof proceeds without modification, and the form of the final bound is unchanged. We therefore state Lemma \ref{lemma:spectral-correclation-decay} under reversibility for clarity of exposition, with the understanding that all subsequent results hold for the general ergodic case by replacing eigenvalues with singular values throughout.


\section{Justification for Norm Bounds Assumption}
\label{appdix:norm-bound-assumption}
\noindent\textbf{Bound on $\|\vv_s\|$.} The value vector is $\vv_s = \mW_V \operatorname{LN}(\vh_s)$, where $\operatorname{LN}$ denotes
layer normalisation.  Layer normalisation standardises each coordinate so that
$\operatorname{LN}(\vh_s)_i = \gamma_i (\vh_{s,i} - \hat\mu)/\hat\sigma + \beta_i$,
where $\hat\mu$ and $\hat\sigma$ are the empirical mean and standard deviation of the entries of $\vh_s$.  After normalisation, the pre-affine vector satisfies
$\sum_i \vz_i^2 = d$, giving
\begin{align}
  \|\operatorname{LN}(\vh_s)\|^2
  &= \sum_i (\gamma_i \vz_i + \beta_i)^2
  \\
  &\leq 2\!\left(\|\gamma\|_\infty^2 d + \|\beta\|^2\right)
  \\
  &\leq C_{\mathrm{LN}}^2 d,
\end{align}
for a constant $C_{\mathrm{LN}}$ depending only on the learned scale $\gamma$ and bias $\beta$.   The projection matrix $\mW_V$ is a fixed trained weight; its operator norm
$\|\mW_V\|_{\mathrm{op}} < \infty$ is a finite model-dependent constant.
Combining we have:
\begin{align}
  \|\vv_s\|
  &= \|\mW_V \operatorname{LN}(\vh_s)\|
  \\
  &\leq \|\mW_V\|_{\mathrm{op}}\, C_{\mathrm{LN}} \sqrt{d}
  =: B_V.
\end{align}

\noindent\textbf{Bound on $\mG_t$}. The gradient $\mG_t = \partial\mathcal{L}/\partial \vo_t$ back-propagates through the
final layer normalisation (with Jacobian $\mJ_{\mathrm{LN}}$) and the language-model
head $\mW_{\mathrm{lm}}$, giving
\begin{align}
  \mG_t = \mJ_{\mathrm{LN}}^T\, \mW_{\mathrm{lm}}^T\,(\vp_t - \vy_t),
\end{align}
where $\vp_t \in \Delta^{|\mathcal{V}|-1}$ is the predicted distribution and $\vy_t \in \{0,1\}^{|\mathcal{V}|}$ is the one-hot target.
By submultiplicativity of the operator norm:
\begin{align}
  \|\mG_t\|
  \leq \|\mJ_{\mathrm{LN}}\|_{\mathrm{op}}\,\|\mW_{\mathrm{lm}}\|_{\mathrm{op}}\,\|\vp_t - \vy_t\|.
\end{align}
Since $\vp_t - \vy_t$ is a difference of probability vectors, $\|\vp_t - \vy_t\| \leq \sqrt{2}$.
For a fixed trained model, both $\|\mJ_{\mathrm{LN}}\|_{\mathrm{op}}$ and
$\|\mW_{\mathrm{lm}}\|_{\mathrm{op}}$ are finite constants.  Therefore:
\begin{align}
  \|\mG_t\|
  \leq \sqrt{2}\,\|\mJ_{\mathrm{LN}}\|_{\mathrm{op}}\,\|\mW_{\mathrm{lm}}\|_{\mathrm{op}}
  =: B_G.
\end{align}

\section{Justification for Score-function Identity Assumption}
\label{appdix:score-func-assumption}
At the optimum $p_\vtheta = p_{\mathrm{data}}$, the
cross-entropy loss with respect to the attention output $\vo_t$ satisfies, for any fixed context $x_{1:t}$:
\begin{align}
  &\mathbb{E}_{y_t}\!\left[
    \frac{\partial}{\partial \vo_t}
    \bigl(-\log p_\theta(y_t \mid x_{1:t})\bigr)
  \right]
  \\
  &= -\frac{\partial}{\partial \vo_t}\sum_{y_t}
    p(y_t \mid x_{1:t})\log p_\vtheta(y_t \mid x_{1:t})
    \bigg|_{p_\theta = p_{\mathrm{data}}}
  \\
  &= 0,
\end{align}
since the cross-entropy gradient vanishes at its own minimum.  Applying the law of total expectation over $x_{1:t}$ then yields $\mathbb{E}[\mG_t] = 0$.  More precisely, because the gradient is computed at a given context $x_{1:t}$ and then averaged
over the target $y_t$ drawn from the conditional, the identity holds conditionally:
\begin{align}
  \mathbb{E}\!\left[\mG_t \mid x_{1:t}\right] = 0.
\end{align}

\section{Preliminary Definition and Lemma}
\label{appdix:preliminary-defs}
We list the following necessary definitions and lemmas for the proof. 
\begin{definition}[$L^2(\pi)$ Space] For a Markov chain with finite state space $\mathcal{Z}$ and stationary distribution $\pi$,
define the weighted function space:
\begin{equation}
    L^2(\pi) \;=\; \bigl\{\,f : \mathcal{Z} \to \mathbb{R}\bigr\}
\end{equation}
with inner product
\begin{align}
    \langle f,g\rangle_\pi \;&:=\; \sum_{z\in\mathcal{Z}} \pi(z)\,f(z)\,g(z) \;\\&=\; \mathbb{E}_\pi[f(Z)\,g(Z)]
\end{align}
and norm $\|f\|_{L^2(\pi)} := \sqrt{\langle f,f\rangle_\pi}$.
\end{definition}

In addition, mean-zero means $\langle f, 1\rangle_{\pi} = \mathbb{E}_{\pi}[f(Z)] = 0$, i.e., the function has zero expectation under $\pi$.

\begin{definition}[Operator $P$ acting on functions]
\label{def:markov-operator}
We define the Markov operator $P$ acting on functions:
\begin{align}
    (Pf)(z) &:= \sum_{z'} P(z,z')\,f(z')
    \\ &= \mathbb{E}[f(Z_{t+1}) \mid Z_t = z]
\end{align}
\end{definition}

\begin{lemma}
For a reversible chain, detailed balance holds: $\pi(z) P(z,z') = \pi(z') P(z',z)$, which makes $P$  symmetric in $L^2(\pi)$.
\end{lemma}
\begin{proof}
    According to the definition of Markov operator $P$ in \ref{def:markov-operator}, we have:
    \begin{align}
        \langle Pf,\,g\rangle_\pi &= \sum_z \pi(z)(Pf)(z)\,g(z) \\
        &= \sum_{z,z'} \pi(z)\,P(z,z')\,f(z')\,g(z)\\
        &= \sum_{z,z'} \pi(z')\,P(z',z)\,f(z')\,g(z)\\
        &= \sum_{z'} \pi(z')\,f(z')\sum_z P(z',z)\,g(z)\\
        &= \langle f,\,Pg\rangle_\pi
    \end{align}
\end{proof}
Since $|\gZ|=K<\infty$, the finite-dimensional spectral theorem for
self-adjoint operators yields a real orthonormal eigenbasis
$\{\varphi_k\}_{k=1}^{K}$ of $\Ltwopi$ with eigenvalues
$1=\lambda_1\ge|\lambda_2|\ge\cdots\ge|\lambda_K|$, and spectral gap $1 - \rho$ where $\rho := \max_{k \ge 2}|\lambda_k|$:
\begin{equation}
    P\varphi_k=\lambda_k\varphi_k,\langle \varphi_k,\varphi_j\rangle_\pi=\delta_{kj}
\end{equation}
Further, for mean-zero $f$:
\begin{equation}
    \langle f, \varphi_1 \rangle_\pi = \langle f, \mathbf{1}\rangle_\pi = \mathbb{E}_\pi[f(Z)] = 0
\end{equation}
which means $f$ has no component along $\psi_1$.  Its expansion uses only $k \ge 2$:
\begin{equation}
    f = \sum_{k \ge 2} a_k \varphi_k, \qquad a_k = \langle f,\varphi_k\rangle_\pi 
\end{equation}

\section{Proof for Spectral Correlation Decay Lemma}
\label{appdix:spectral-correlation-decay-proof}
\begin{proof}
At stationarity ($z_s \sim \pi$), the Tower Property tells:
    \begin{equation}
        \mathbb{E}[f(z_t)\,g(z_s)]
\;=\; \mathbb{E}\bigl[\mathbb{E}[f(z_t)\mid z_s]\cdot g(z_s)\bigr]
    \end{equation}
    Note that the inner expectation is the $\tau$-step prediction: for a fixed starting state $z_s = z$,
the expected value of $f(z_t)$ after $\tau = t-s$ steps is exactly $(P^{\tau} f)(z)$:
\begin{equation}
    (P^{\tau} f)(z) = \mathbb{E}[f(z_t) \mid z_s = z]
\end{equation}
Therefore:
\begin{align}
    \mathbb{E}[f(z_t)\,g(z_s)]
& = \mathbb{E}[(P^{\tau} f)(z_s)\,g(z_s)]\\
&= \sum_z \pi(z)\,(P^{\tau} f)(z)\,g(z)
\\ &= \langle P^{\tau} f,\,g\rangle_\pi
\end{align}

Then we expand $f,g$ by the orthonormal eigenbasis:
\begin{equation}
    f = \sum_{k \ge 2} a_k \varphi_k, g = \sum_{k \ge 2} b_k \varphi_k,
\end{equation}

Using the orthonormality $\langle \varphi_k,\varphi_j\rangle_\pi=\delta_{kj}$, we further simplify $\langle P^{\tau} f,\,g\rangle_\pi$:
\begin{align}
    \langle P^{\tau} f,\,g\rangle_\pi &= \left\langle \sum_{k \ge 2} a_k\lambda_k^{\tau} \varphi_k,\; \sum_{j \ge 2} b_j\varphi_j\right\rangle_\pi \\
    &=\sum_{k \ge 2} a_k\,\lambda_k^{\tau}\,b_k
\end{align}

Then we bound it:
\begin{align}
    \bigl|\langle P^{\tau} f, g\rangle_\pi\bigr| &= \left|\sum_{k \ge 2} a_k\lambda_k^{\tau} b_k\right| \\
    &\le \sum_{k \ge 2}|a_k|\,|\lambda_k|^{\tau}\,|b_k|\\
&\le \rho^{\tau} \sum_{k \ge 2} |a_k|\,|b_k|
\end{align}

Applying the Cauchy-Schwarz Inequality:
\begin{align}
    \sum_{k \ge 2} |a_k|\,|b_k| &\leq  \sqrt{\sum_{k \ge 2}|a_k|^2}\sqrt{\sum_{k \ge 2}|b_k|^2}\\
    &=||f||_{L^2(\pi)}||g||_{L^2(\pi)}
\end{align}

Now combining all together:
\begin{align}
    \bigl|\mathbb{E}[f(z_t)\,g(z_s)]\bigr| &\leq \rho^{\tau} ||f||_{L^2(\pi)}||g||_{L^2(\pi)}\\
    &= \rho^{\,t-s}\,\|f\|_{L^2(\pi)}\,\|g\|_{L^2(\pi)}
\end{align}
\end{proof}

\section{Proof of the Gradient Locality Theorem}
\label{appdix:proof-gradient-locality}
\begin{proof}
For clarity, we drop the layer/head superscripts in the following derivation. First, we bound the gradient:
\begin{align}
    \left|\frac{\partial\mathcal{L}}{\partial e_{ts}}\right| \;\le\; |\mG_t\cdot \vv_s| + |\mG_t\cdot\bar{\vv}_t|
\end{align}

Then we consider how to bound $\mathbb{E}[|\mG_t\cdot \vv_s|]$. Applying the Tower Property:
\begin{align}
   \mathbb{E}[\mG_t\cdot \vv_s]&= \mathbb{E}[\mathbb{E}[\mG_t|z_t]\cdot \mathbb{E}[\vv_s|z_s]]\\
   &=\mathbb{E}[\phi(z_t)\cdot\psi(z_s)]
\end{align}
where we define:
\begin{align}
    \phi(z)=\mathbb{E}[\mG_t|z_t=z], \psi(z)=\mathbb{E}[\vv_s|z_s=z]
\end{align}

According to Assumption \ref{asm:score}:
\begin{align}
    \mathbb{E}_{\pi}[\phi(z)] &= \sum_z \pi(z)\phi(z) \\
    &= \sum_zP(z_t=z)\phi(z)\\
    &=\mathbb{E}[\phi(z_t)] = \mathbb{E}[\mathbb{E}[\mG_t|z_t=z]]\\
    &=\mathbb{E}[\mG_t] =\mathbb{E}[[\mG_t \big|x_{1:t}]] = \vzero
\end{align}
Similarly, centering $\psi$ does not change the covariance (since $\mathbb{E}[\phi(z_t)] = 0$).  Both $\phi$ and $\psi$ are mean-zero functions
on $\mathcal{Z}$.

Then applying Lemma \ref{lemma:spectral-correclation-decay}:
\begin{align}
    \bigl|\mathbb{E}[\mG_t\cdot \vv_s]\bigr| &= \bigl|\mathbb{E}[\phi(z_t)\cdot\psi(z_s)]\bigr|\\
    &\le \rho^{\tau} ||\phi||_{L^2(\pi)}||\psi||_{L^2(\pi)}\\
    &\le B_GB_V\rho^{\tau}
\end{align}
Applying the same bound to the $\bar{\vv}_t$ term, which satisfies $\|{-}\bar{\vv}_t\| \le B_V$
as a convex combination:
\begin{align}
    g(\tau) = \mathbb{E}\!\left[\left|\frac{\partial\mathcal{L}}{\partial e_{ts}}\right|\right]
\;\le\; 2B_G B_V\,\rho^{\tau} = C'\,e^{-\tau/W},
\end{align}
where $C' = 2B_G B_V$ and $W = \frac{-1}{\log\rho}$.
\end{proof}

\section{Proof for Token Autocorrelation Decay}
\label{appdix:token-autocorrelation-proof}
Define the centred indicator feature $f_v(x) := \mathbf{1}[x=v] - \pi(v)$ and the
$\pi$-weighted autocorrelation at lag $\tau$:
\begin{align}
    &\bar\rho(\tau):= \sum_v \pi(v)\frac{\mathrm{Cov}(f_v(x_t),\,f_v(x_{t+\tau}))}{\mathrm{Var}(f_v(x_t))}
\end{align}
For the numerator $\mathrm{Cov}(f_v(x_t),\,f_v(x_{t+\tau}))$, we expand it as:
\begin{equation}
    \mathbb{E}[f_v(x_t)\,f_v(x_{t+\tau})] -\mathbb{E}[f_v(x_t)]\mathbb{E}[f_v(x_{t+\tau})]
\end{equation}

Considering the expectation of $f_v(x)$:
\begin{align}
    \mathbb{E}[f_v(x_t)] &= \mathbb{E}[\mathbf{1}[x_t=v] - \pi(v)]\\
    &=\pi(v)-\pi(v)=0
\end{align}
Therefore we can further simplify:
\begin{align}
    &\bar\rho(\tau)= \sum_v \pi(v)\frac{\mathbb{E}[f_v(x_t)\,f_v(x_{t+\tau})]}{\pi(v)(1-\pi(v))}
\end{align}

In the HMM, tokens are conditionally independent given their latent states:
$x_t \perp x_{t+\tau} \mid z_t,\, z_{t+\tau}$.  By the tower property:
\begin{align}
    &\mathbb{E}[f_v(x_t)\,f_v(x_{t+\tau})]\\
&= \mathbb{E}\!\bigl[\,\mathbb{E}[f_v(x)\mid z_t]\cdot\mathbb{E}[f_v(x)\mid z_{t+\tau}]\bigr]
\\ &= \mathbb{E}[\phi_v(z_t)\cdot\phi_v(z_{t+\tau})]
\end{align}
where $\phi_v(z)=\mathbb{E}[f_v(x)|z]=P(x=v|z)-\pi(v)$ is a mean-zero function of $z$:
\begin{align}
    \mathbb{E}_{\pi}[\phi_v(z)]&=\mathbb{E}_{\pi}[P(x=v|z)-\pi(v)]\\
    &=\sum_z\pi(z)P(x=v|z)-\pi(v)\\
    &=\pi(v)-\pi(v)=0
\end{align}

Since both $\phi_v(z_t)$ and $\phi_v(z_{t+\tau})$ are mean-zero, we can apply Lemma \ref{lemma:spectral-correclation-decay} directly:
\begin{align}
    \bigl|\mathbb{E}[\phi_v(z_t)\cdot\phi_v(z_{t+\tau})]\bigr|
\;\le\; \rho^{\tau}\,\|\phi_v\|^2_{L^2(\pi)}
\end{align}
where $\rho = \lambda_2$ is the second-largest eigenvalue of the latent transition
matrix $P$. Now we can bound the autocorrelation:
\begin{align}
    \bar\rho(\tau)
&= \sum_v \pi(v)\;\frac{\mathbb{E}[f_v(x_t)\,f_v(x_{t+\tau})]}{\pi(v)(1-\pi(v))} \\ &\le\; \rho^{\tau} \sum_v \frac{\|\phi_v\|^2_{L^2(\pi)}}{\pi(v)(1-\pi(v))}
\;= C\cdot\rho^{\tau}
\end{align}

\noindent\textbf{Conclusion.} Observing $\bar\rho(\tau) \sim C\,\hat\rho^{\tau}$ with $\hat\rho < 1$ in the corpus confirms that the latent chain has spectral gap
$1 - \hat\rho > 0$, which is exactly Lemma \ref{lemma:spectral-correclation-decay}'s hypothesis. 

\section{Discussion of Attention Gradient Decay in Pre-training Stage}
\label{appdix:discussion-gen-decay}
  Theorem \ref{thm:gradient_locality} rests on Lemma \ref{lemma:spectral-correclation-decay}, which holds for any ergodic Markov chain with a spectral gap: a property of the data-generating process, not specific to CoT structure, which is also observed in general pre-training data corpus.

  Consequently, the gradient locality bound $g(\tau) \le C' e^{-\tau/W}$ applies to any training corpus, including general pre-training data. The severity of routing collapse is governed by the mismatch $\Delta W = W_\text{corr} - W_\text{grad}$: how far the data demands the model to attend vs. how far the gradient actually reinforces it. For general pre-training corpora, this mismatch is smaller than for CoT, so routing degrades more slowly but by the same mechanism. 
  
  For hybrid architectures, however, even the slower pre-training erosion meaningfully reduces available recall capacity before fine-tuning begins, leaving the model closer to the threshold.  The hybrid architecture's vulnerability therefore originates during pre-training and is further amplified by CoT fine-tuning.

\section{Proof for the Routing-Extraction Gradient Decoupling Theorem}
\label{appdix:proof-re-gradient-decouple}
\begin{proof}
From \autoref{eq:wq_grad}, we have:
\begin{align}
    \nabla_{\mW^{Q,(h)}}\mathcal{L}
= \frac{1}{\sqrt{d_h}}\sum_{t,s} \frac{\partial\mathcal{L}}{\partial e_{ts}^{(h)}}\, \vk_s^{(h')}\, \vh_t^T
\end{align}
where $h'$ is the corresponding KV head. Then considering its Frobenius norm:
\begin{align}
    \mathbb{E}\!\left[\left\|\frac{\partial \mathcal{L}}{\partial e_{ts}}\, \vk_s\, \vh_t^T\right\|_F\right]
&= \mathbb{E}\!\left[\left|\frac{\partial \mathcal{L}}{\partial e_{ts}}\right|\right]\cdot\|\vk_s\|\cdot\|\vh_t\| \\
&\le\; C' B_K B_h\, \rho^{\tau}
\\ &=C_R\cdot \rho^{\tau}
\end{align}
via Theorem \ref{thm:gradient_locality} and bounded norms $\|\vk_s\| \le B_K$, $\|\vh_t\| \le B_h$.

For the Exctraction Parameters, the value vector $\vv_s = \mW_V \vh_s$ enters the loss through $\vo_t = \sum_s A_{ts} \vv_s$. Accumulating all downstream dependencies:
\begin{align}
    \frac{\partial \mathcal{L}}{\partial \vv_s}
= \sum_{t \ge s} A_{ts}\, \mG_t
\end{align}
Now we assume that $A_{ss} \ge \delta_A > 0$ for all $s$ and $\mathbb{E}[\|\mG_s\|] \ge c_G > 0$ for all $s$:
\begin{align}
    \mathbb{E}\!\left[\left\|\frac{\partial \mathcal{L}}{\partial \vv_s}\right\|\right]\ge\mathbb{E}[A_{ss}\|\mG_s\|] \ge\delta_A c_G > 0
\end{align}
\end{proof}

\section{Training Data and Configuration Details}
\label{appdix:training-details}
For HypeNet, we conduct both the pre-training and SFT.

\noindent\textbf{Pre-training Data.} Since we mainly focus on the long-context recall and math reasoning ability, during the pre-training stage, we mix the FineWeb-Edu \citep{lozhkov2024fineweb-edu} and UltraData-Math \citep{ultradata-math}. Following observations from recent pretraining studies that maintain web corpora as the dominant source while increasing structured reasoning data \citep{allal2025smollm2smolgoesbig}, we construct a mixture with $80\%$ general web data and $20\%$ math-focused data. Prior evidence indicates that web-scale corpora preserve broad linguistic and factual competence, while specialized mathematical corpora improve reasoning and STEM performance.

\noindent\textbf{Pre-training Config.} We adopt the 3-stage training for HypeNet, keeping the same setting in \citet{chen2026hybridlinearattentionright}, shown in \autoref{tab:training-config-2b}, \autoref{tab:training-config-5b}, and \autoref{tab:training-config-9b}.

\begin{table*}[!t]
\centering
\caption{Pre-Training Configs for HypeNet-2B} 
\label{tab:training-config-2b}
\resizebox{\linewidth}{!}{%
\begin{tabular}{l|cccccc}
\toprule
\textbf{\textsc{Stage}}                  & \textbf{\textsc{Tokens}}             & \textbf{\textsc{LR}} & \textbf{\textsc{LR Scheduler}}            & \textbf{\textsc{Context Len.}}           & \textbf{\textsc{Batch}} & \textbf{\textsc{Training Steps}}                \\ \midrule
1 & 320M & 1e-3 $\rightarrow$ 1e-5 & Cosine & 512 & 32 & 20,000 \\
2 & 1B & 1e-4 $\rightarrow$ 1e-5 & Cosine & 512 & 96 & 20,000 \\
3 & 1B & 1e-5 & Constant & 16,384 & 128 & 500
\\ \bottomrule
\end{tabular}
}
\end{table*}

\begin{table*}[!t]
\centering
\caption{Pre-Training Configs for HypeNet-5B} 
\label{tab:training-config-5b}
\resizebox{\linewidth}{!}{%
\begin{tabular}{l|cccccc}
\toprule
\textbf{\textsc{Stage}}                  & \textbf{\textsc{Tokens}}             & \textbf{\textsc{LR}} & \textbf{\textsc{LR Scheduler}}            & \textbf{\textsc{Context Len.}}           & \textbf{\textsc{Batch}} & \textbf{\textsc{Training Steps}}                \\ \midrule
1 & 320M & 1e-3 $\rightarrow$ 1e-5 & Cosine & 512 & 32 & 20,000 \\
2 & 1B & 5e-5 $\rightarrow$ 1e-5 & Cosine & 512 & 96 & 20,000 \\
3 & 1B & 1e-5 & Constant & 16,384 & 128 & 500
\\ \bottomrule
\end{tabular}
}
\end{table*}

\begin{table*}[!t]
\centering
\caption{Pre-Training Configs for HypeNet-9B} 
\label{tab:training-config-9b}
\resizebox{\linewidth}{!}{%
\begin{tabular}{l|cccccc}
\toprule
\textbf{\textsc{Stage}}                  & \textbf{\textsc{Tokens}}             & \textbf{\textsc{LR}} & \textbf{\textsc{LR Scheduler}}            & \textbf{\textsc{Context Len.}}           & \textbf{\textsc{Batch}} & \textbf{\textsc{Training Steps}}                \\ \midrule
1 & 320M & 1e-3 $\rightarrow$ 1e-5 & Cosine & 512 & 32 & 20,000 \\
2 & 1B & 3e-5 $\rightarrow$ 1e-5 & Cosine & 512 & 96 & 20,000 \\
3 & 1B & 1e-5 & Constant & 16,384 & 128 & 500
\\ \bottomrule
\end{tabular}
}
\end{table*}

\begin{table}[!t]
\centering
\caption{SFT Configs for 2B scale model} 
\label{tab:sft-config-2b}
\resizebox{0.7\linewidth}{!}{%
\begin{tabular}{ccccccc}
\toprule
\textbf{\textsc{LR}} & \textbf{\textsc{LR Scheduler}}            & \textbf{\textsc{Context Len.}}           & \textbf{\textsc{Batch}} & \textbf{\textsc{Training Steps}}                \\ \midrule
1e-5 & Constant & 16,384 & 128 & 100
\\ \bottomrule
\end{tabular}
}
\end{table}

\begin{table}[!t]
\centering
\caption{SFT Configs for 5B, 9B scale model} 
\label{tab:sft-config-5b}
\resizebox{0.7\linewidth}{!}{%
\begin{tabular}{ccccccc}
\toprule
\textbf{\textsc{LR}} & \textbf{\textsc{LR Scheduler}}            & \textbf{\textsc{Context Len.}}           & \textbf{\textsc{Batch}} & \textbf{\textsc{Training Steps}}                \\ \midrule
1e-5 $\rightarrow$ 1e-6 & Cosine & 16,384 & 128 & 100
\\ \bottomrule
\end{tabular}
}
\end{table}

\noindent\textbf{SFT Data.} To further improve the model's math reasoning performance, we conduct CoT-SFT with MiroMind-M1 dataset \citep{li2025miromindm1opensourceadvancementmathematical}, which is collected from OpenR1 \citep{openr1}, Open-thoughts \citep{guha2025openthoughts}, Light-R1 \citep{wen2025light}, and Synthetic-1 \citep{2025synthetic1}.
\noindent\textbf{SFT Config.} The detailed configurations are shown in 
\autoref{tab:sft-config-2b} and \autoref{tab:sft-config-5b}.

\section{Evaluation Details}
\label{appdix:eval-details}
For evaluating NIAH, we apply LM Evaluation Harness\footnote{\url{https://github.com/EleutherAI/lm-evaluation-harness}} \citep{eval-harness} for official test. For evaluating MATH500 and GSM8K, we report average \texttt{pass@1} over 16 independent generations and \texttt{Maj@16} (i.e., majority vote as the final prediction) per problem. In detail, for MATH500, the generation length is 8,192; for GSM8K, we set it as 2,048, since it is easier. For LiveCodeBench, we report the average \texttt{pass@1} over 8 independent generations with 16,384 response length. To measure the model's performance after fine-tuning on Tulu3, we report the Prompt-level strict accuracy on IFEval task \citep{zhou2023instructionfollowingevaluationlargelanguage}.

\section{Analysis on Pure Softmax-Attention Model}
Our main experiments concentrate on the hybrid models, and in this section, we explore whether SFT can degrade the pure softmax-attention models as well. We include Qwen2.5-3B, Qwen2.5-7B \citep{qwen2025qwen25technicalreport} and Misrtral-7B-Instruct-v0.3 \citep{jiang2023mistral7b} for investigation. We include the details on how we identify the top layers for retrieval in 
\autoref{appdix:iden-retrieve-layer}.
\begin{table}[!t]
\centering
\caption{NIAH on pure softmax-attention models.} 
\label{tab:ablation-full-attn}
\resizebox{0.7\linewidth}{!}{%
\begin{tabular}{ll|llll}
\toprule
\textbf{\textsc{Model}}   & \textbf{\textsc{Method}} & \multicolumn{2}{c}{\textbf{\textsc{NIAH-Single-2}}}    & \multicolumn{2}{c}{\textbf{\textsc{NIAH-Single-3}}}  \\
& &  \textbf{32K} & \textbf{64K}  & \textbf{32K} & \textbf{64K} 
\\ \midrule
Qwen2.5-7B & Pre-train & 100.0 & 95.8 & 99.8 & 98.6\\
& +SFT & 100.0 & 92.6 & 99.8 & 99.4\\ \midrule

Qwen2.5-3B & Pre-train & 100.0 & 94.6 & 100.0 & 99.6\\
& +SFT & 100.0 & 93.8 & 100.0 & 90.8\\
& +\textbf{\method} & 100.0 & 93.8 & 100.0 & 93.4\\ \midrule

Mistral-7B & Pre-train & 99.8 & 90.2 & 99.8 & 81.0 \\
& +SFT & 99.6 & 59.6 & 65.8 & 4.80 \\
& +\textbf{\method} & 99.8 & 62.0 & 70.0 & 5.40 

\\\bottomrule
\end{tabular}
}
\end{table}

\noindent\textbf{Analysis.}  \autoref{tab:ablation-full-attn} shows that for Mistral-7B, it exhibits strong sensitivity to SFT at longer contexts. SFT causes severe degradation on NIAH-Single-3 at 64K, dropping from $81.0$ to $4.80$, suggesting a breakdown of long-range retrieval. \method recovers this degradation limitedly, indicating weaker long-context representations in the pre-trained model.
\begin{figure}[!ht]
    \centering
    \includegraphics[width=0.55\linewidth]{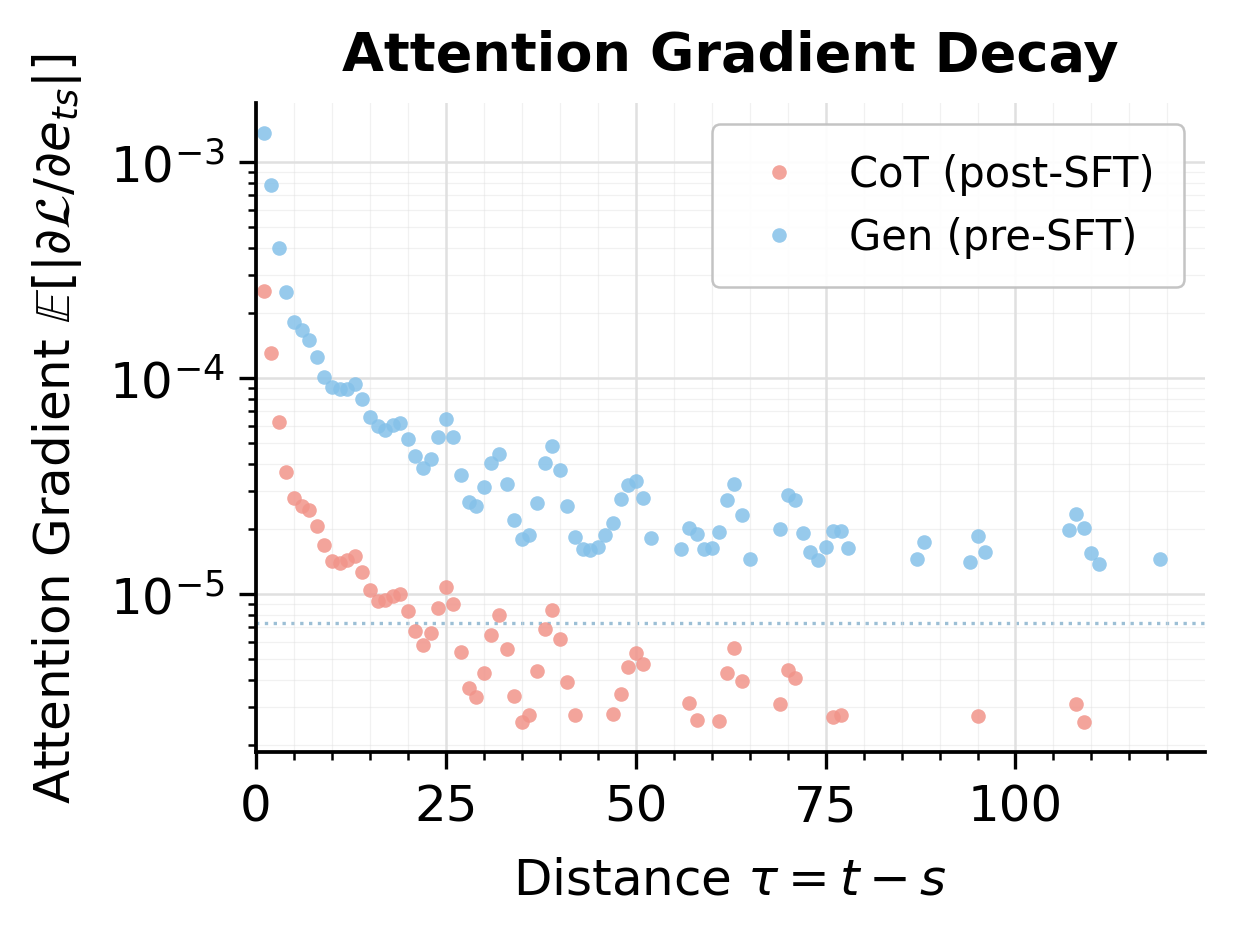}
    \vspace{-2mm}
    \caption{Attention gradient decay on Qwen2.5-3B. We observe a different pattern from the exponential decay of HypeNet.}
    \label{fig:qwen-attention-decay}
\end{figure}

For Qwen2.5-7B and Qwen2.5-3B, long-context performance is highly stable across NIAH-Single-2 and Single-3. SFT induces only minor fluctuations (e.g., Qwen2.5-7B drops from $95.8$ to $92.6$ at 64K), indicating that these models already learn robust long-context retrieval during pre-training. Accordingly, \method yields only marginal gains, consistent with a near-saturation regime where little capability is lost. To investigate the underlying mechanism, we measure the attention gradient decay on Qwen2.5-3B, shown in \autoref{fig:qwen-attention-decay}. Unlike HypeNet-2B, whose full-attention layers exhibit a well-fitted exponential decay, it has a clear near-flat tail. We hypothesize that long-context recall in pure softmax-attention models is maintained structurally, via the abundance of full-attention layers.

\section{Method of Identifying Retrieval Layers in Pure Softmax-Attention Models}
\label{appdix:iden-retrieve-layer}
To identify which softmax-attention layers are primarily responsible for long-range retrieval, we conduct a leave-one-out ablation study over the set ${L}_{\mathrm{attn}}$. For each layer $\ell \in {L}_{\mathrm{attn}}$, we construct an ablated model $\mathcal{M}^{(-\ell)}$ by setting $\mW_Q^\ell = \mW_K^\ell = \mathbf{0}$. Each ablated model is evaluated on several retrieval benchmarks (Ruler-QA-SQuAD \citep{hsieh2024ruler}, FDA \citep{arora2023language}, SWDE \citep{lockard2019openceres}). The importance score of layer $\ell$ is defined as:
\begin{equation*}
    \delta\text{Recall}(\ell) = \mathcal{S}(\mathcal{M}) - \mathcal{S}(\mathcal{M}^{(-\ell)}),
\end{equation*}
where $\mathcal{S}(\cdot)$ denotes the average retrieval accuracy across benchmarks. A large $\delta\text{Recall}(\ell)$ indicates that layer $\ell$ contributes disproportionately to retrieval and that its routing weights are critical; a near-zero score indicates redundancy. Layers are ranked by $\delta\text{Recall}$ in descending order, and the top-$k$ layers are selected as the target set for weight restoration. In our experiments, we keep $k=L_{\mathrm{attn}}/4$ as the same as in HypeNet.

\section{Details of deriving \textsc{QK-Pro}}
\label{appdix:graft-pro}
The bilinear constraint $\mW_Q^{\rm new} \mW_{K}^{\rm new \top} = \mR_{\rm pre}$ has no closed-form joint solution, so we linearise by fixing one factor to its pre-SFT value, making the remaining problem a standard constrained least-squares system solvable via Lagrange multipliers.

We fix $\mW_K^{\rm new} = \mW_K^{\rm pre}$, find minimum-deviation $\mW_Q^{\rm new}$:
\begin{align}
    \min_{\mW_Q^{\rm new}} \left\|\mW_Q^{\rm new} - \mW_Q^{\rm post}\right\|_F  
    \\\text{s.t.} \quad \mW_Q^{\rm new} \mW_{K}^{\rm new \top} = \mR_{\rm pre}
\end{align}
Then applying the Lagrange multipliers:
{\small
\begin{align}
    \mathcal{L}=\left\|\mX - \mW_Q^{\rm post}\right\|_F+ \mathrm{tr}(\Lambda^{\top}(\mX \mW_{K}^{\rm new \top} -\mR_{\rm pre}))
\end{align}
}
Taking the derivative $\partial\mathcal{L}/\partial \mX = 0$, and we get:
\begin{equation}
\begin{adjustbox}{max width=0.99\linewidth}
$
\begin{aligned}
\mW_Q^{\rm new} = \mW_Q^{\rm post} + \bigl(\mR_{\rm pre} - \mW_Q^{\rm post} \mW_{K}^{\rm pre\top}\bigr)\bigl(\mW_K^{\rm pre} \mW_{K}^{\rm pre\top} + \lambda \mI\bigr)^{-1} \mW_K^{\rm pre}
\end{aligned}
$
\end{adjustbox}
\end{equation}

\end{document}